\documentclass[11pt]{article}
\pdfoutput=1

\usepackage{fullpage}
\usepackage{amsmath}
\usepackage{amssymb}
\usepackage{graphicx} 

\usepackage{fancybox,graphicx}
\usepackage[T1]{fontenc}
\usepackage[latin1]{inputenc}
\usepackage{float}
\usepackage{subfigure}
\usepackage{epsfig}
 \floatstyle{ruled}
\newfloat{alg}{tbp}{loa}
\floatname{alg}{Algorithm}

{
\rm
\begin{tabbing}
....\=...\=...\=...\=...\=  \+ \kill
} %
{\end{tabbing}
}

\newtheorem{definition}{Definition}[section]
\newtheorem{lemma}{Lemma}[section]
\newtheorem{corollary}{Corollary}[section]
\newtheorem{theorem}{Theorem}[section]
\newtheorem{proposition}{Proposition}[section]
\newtheorem{assumption}{Assumption}[section]

\newcommand{\bx}{{\mathbf x}}

\newcommand{\by}{{\mathbf y}}

\newcommand{\rE}{{\mathbb E}}

\newcommand{\pr}{\mathrm{Pr}}

\newcommand{\Fr}{{\mathrm{supp}}}

\newcommand{\hb}{{\hat{\beta}}}
\newcommand{\bb}{{\bar{\beta}}}
\newcommand{\R}{{\mathbb R}}

\newcommand{\BlackBox}{\rule{1.5ex}{1.5ex}}
\newenvironment{proof}{\par\noindent{\bf Proof\ }}{\hfill\BlackBox\\[2mm]}

\begin{document}

\title{The Benefit of Group Sparsity}
\author{Junzhou Huang \\ Computer Science Department, Rutgers University 
\and Tong Zhang \\ Statistics Department, Rutgers University}
\date{}
\maketitle

\begin{abstract}
This paper develops a theory for group Lasso using
a concept called {\em strong group sparsity}.
Our result shows that group Lasso is superior to standard Lasso
for strongly group-sparse signals.
This provides a convincing theoretical justification for using
group sparse regularization when the underlying group structure is consistent with the data. Moreover, the theory predicts some limitations of the group Lasso formulation that are confirmed by simulation studies.
\end{abstract}

\section{Introduction}

We are interested in the sparse learning problem for
least squares regression.
Consider a set of $p$ basis vectors $\{\bx_1,\ldots,\bx_p\}$
where $\bx_j \in \R^n$ for each $j$. Here, $n$ is the sample size.

Denote by $X$ the $n \times p$ data matrix, with column $j$
of $X$ being $\bx_j$.
Given an observation $\by=[y_1,\ldots,y_n] \in \R^n$
that is generated from a sparse linear combination
of the basis vectors plus a stochastic noise
vector $\epsilon \in \R^n$:
\[
\by = X \bb + \epsilon = \sum_{j=1}^d \bb_j \bx_j + \epsilon,
\]
where we assume that the target coefficient $\bb$ is sparse.
Throughout the paper, we consider fixed design only. That is,
we assume $X$ is fixed, and randomization is with respect
to the noise $\epsilon$. Note that we do not assume that the noise
$\epsilon$ is zero-mean.

Define the support of a sparse vector $\beta \in \R^p$ as
\[
\Fr(\beta)=\{j: \beta_j \neq 0\} ,
\]
and $\|\beta\|_0 = |\Fr(\beta)|$. A natural method for sparse learning
is $L_0$ regularization:
\[
\hb_{L0}= \arg\min_{\beta \in \R^p}
\|X \beta - \by\|_2^2 \quad \text{subject to } \|\beta\|_0 \leq k,
\]
where $k$ is the sparsity.
Since this optimization problem is
generally NP-hard, in practice, one often consider the following
$L_1$ regularization problem, which is the closest convex
relaxation of $L_0$:
\[
\hb_{L1}= \arg\min_{\beta \in \R^p}
\left[ \frac{1}{n}\|X \beta - \by\|_2^2 + \lambda \|\beta\|_1 \right] ,
\]
where $\lambda$ is an appropriately chosen regularization parameter. This method
is often referred to as Lasso in the statistical literature.

In practical applications, one often knows a group structure on the coefficient
vector $\bar{\beta}$ so that variables in the same group tend to be zeros or
nonzeros simultaneously. The purpose of this paper is to show that if such a
structure exists, then better results can be obtained.

\section{Strong Group Sparsity}

For simplicity, we shall only consider non-overlapping groups in this paper,
although our analysis can be adapted to handle moderately
overlapping groups.

Assume that $\{1,\ldots,p\}=\cup_{j=1}^m G_j$ is partitioned into $m$
disjoint groups $G_1, G_2, \ldots, G_m$:
$G_i \cap G_j = \emptyset$ when $i \neq j$.
Moreover, throughout the paper, we let
$k_j=|G_j|$, and $k_0=\max_{j \in \{1,\ldots,m\}} k_j$.
Given $S \subset \{1,\ldots,m\}$ that denotes a set of groups, we define
$G_S= \cup_{j \in S} G_j$.

Given a subset of variables $F \subset \{1,\ldots,p\}$ and a 
coefficient vector $\beta \in \R^p$,
let $\beta_{F}$ be the vector in $\R^{|F|}$ which is identical to $\beta$
in $F$. Similar, $X_F$ is the $n \times |F|$ matrix
with columns identical to $X$ in $F$.

The following method, often referred to
as group Lasso, has been proposed to take advantage of the group structure:
\begin{equation}
\hb= \arg\min_{\beta}
\left[ \frac{1}{n} \left\|X\beta -\by\right\|_2^2
+ \lambda \sum_{j=1}^m \|\beta_{G_j}\|_2
\right] .
\label{eq:group-lasso}
\end{equation}
The purpose of this paper is to develop a theory
that characterizes the performance of (\ref{eq:group-lasso}).
We are interested in conditions under which group Lasso yields better
estimate of $\bb$ than the standard Lasso.

Instead of the standard sparsity assumption, where the complexity is
measured by the number of nonzero coefficients $k$,
we introduce the strong group sparsity concept
below. The idea is to measure the complexity of a sparse signal
using group sparsity in addition to coefficient sparsity.
\begin{definition}\label{def:sparsity}
  A coefficient vector $\bb \in \R^p$ is $(g,k)$ strongly group-sparse if
  there exists a set $S$ of groups such that
  \[
  \Fr(\bb) \subset G_S , \quad
  |G_S| \leq k , \quad |S| \leq g .
  \]
\end{definition}

The new concept is referred to as strong group-sparsity because
$k$ is used to measure the sparsity of $\bb$ instead of $\|\bb\|_0$.
If this notion is beneficial, then $k/\|\bb\|_0$ should be small, which means
that the signal has to be efficiently covered by the groups.
In fact, the group Lasso method
does not work well when $k/\|\bb\|_0$ is large.
In that case, the signal is only {\em weak group sparse}, and
one needs to use $\|\bb\|_0$ to precisely measure the real
sparsity of the signal.
Unfortunately, such information is not included in the group Lasso formulation, and there is no simple fix of this problem using variations of group Lasso.
This is because our theory requires that the group Lasso regularization term is strong enough to dominate the noise, and the strong regularization causes a
bias of the order $O(k)$ which cannot be removed.
This is one fundamental drawback which is inherent to
the group Lasso formulation.

\section{Related Work}

The idea of using group structure to achieve better sparse recovery
performance has received much attention.
For example, group sparsity has been considered for simultaneous sparse
approximation \cite{Wipf07TSP} and multi-task compressive sensing
 \cite{Ji08TSP} from the Bayesian hierarchical modeling point of view.
Under the Bayesian hierarchical model
framework, data from all sources contribute to the estimation of hyper-parameters in the sparse prior model. The
shared prior can then be inferred from multiple sources.
Although the idea can be justified using standard Bayesian intuition, there are
no theoretical results showing how much better (and under what kind of
conditions) the resulting algorithms perform.

In \cite{Stojnic08}, the authors attempted to derive a bound on the
number of samples needed to recover block sparse signals,
where the coefficients in each block are either all zero or all nonzero.
In our terminology, this corresponds to the case of group sparsity
with equal size groups. The algorithm considered there is a special
case of (\ref{eq:group-lasso}) with $\lambda_j \to 0^+$.
However, their result is very loose, and does not demonstrate the advantage
of group Lasso over standard Lasso.

In the statistical literature, the group Lasso (\ref{eq:group-lasso})
has been studied by a number of authors
\cite{Yuan06JRSS,Bach08-groupLasso,NarRin08,KolYua08,ObWaJo08}.
There were no theoretical results in \cite{Yuan06JRSS}.
Although some theoretical results were developed in
\cite{Bach08-groupLasso,NarRin08}, neither showed that
group Lasso is superior to the standard Lasso.

The authors of \cite{KolYua08} showed that group Lasso can be superior
to standard Lasso when each group is an infinite dimensional kernel,
by using an argument completely different from ours (they relied on the fact that meaningful analysis can be obtained for kernel
methods in infinite dimension). Their idea cannot be adapted to
show the advantage of group Lasso in finite dimensional scenarios
of interests such as in the standard compressive sensing setting.
Therefore our analysis, which focuses on the latter, is complementary to
their work.

Another related work is \cite{ObWaJo08}, where the authors
considered a special case of group Lasso in the multi-task learning
scenario, and showed that the number of samples required for recovering
the exact support set may be smaller for group Lasso under appropriate
conditions. However, there are major differences between our analysis
and their analysis.
For example, the group formulation we consider here is more general
and includes the multi-task scenario as a special case.
Moreover, we study signal recovery performance in
2-norm instead of the exact recovery of support set in their analysis.
The sparse eigenvalue condition employed in this work is often considerably
weaker than the irrepresentable type condition in their analysis
(which is required for exact support set recovery). Our analysis 
also shows that for strongly group-sparse signals, even when the
number of samples is large, the group Lasso can still have advantages
in that it is more robust to noise than standard Lasso.

In the above context, the main contribution of this work is the
introduction of the strong group sparsity concept, under which
a satisfactory theory of group Lasso is developed. Our result
shows that strongly group sparse signals can be estimated more
reliably using group Lasso, in that it requires fewer number of
samples in the compressive sensing setting, and is more robust
to noise in the statistical estimation setting.

Finally, we shall mention that independent of the authors, results 
similar to those presented in this paper have also been obtained in 
\cite{LMTG09} with a similar technical analysis. However, while our
paper studies the general group Lasso formulation, only the special
case of  multi-task learning is considered in \cite{LMTG09}.

\section{Assumptions}

The following assumption on the noise is important in our analysis.
It captures an important advantage of group Lasso over standard Lasso
under the strong group sparsity assumption.
\begin{assumption}[Group noise condition]
\label{assump:noise}
There exist non-negative constants $a, b$
such that for any fixed group $j \in \{1,\ldots,m\}$,
and $\eta \in (0,1)$: with probability larger than
$1-\eta$,
the noise projection to the $j$-th group
is bounded by:
\[
\|(X_{G_j}^\top X_{G_j})^{-0.5} X_{G_j}^\top (\epsilon- \rE \epsilon) \|_2
 \leq a \sqrt{k_j} + b \sqrt{-\ln \eta} .
\]
\end{assumption}

The importance of the assumption is that the concentration term
$\sqrt{-\ln \eta}$ does not depend on $k$. This reveals a significant benefit
of group Lasso over standard Lasso: that is, the concentration term
does not increase when the group size increases. This implies that if
we can correctly guess the group sparsity structure, the group
Lasso estimator is more stable with respect to stochastic
noise than the standard Lasso.

We shall point out that this assumption holds for independent
sub-Gaussian noise vectors, where $e^{t (\epsilon_i-\rE \epsilon_i)} \leq e^{t^2 \sigma^2/2}$ 
for all $t$ and $i=1,\ldots,n$. It can be shown that
one may choose $a=2.8$ and $b=2.4$ when $\eta \in (0,0.5)$.
Since a complete treatment of sub-Gaussian noise is not
important for the purpose of this paper,
we only prove this assumption under independent
Gaussian noise, which can be directly calculated.
\begin{proposition} \label{prop:gaussian}
Assume the noise vector $\epsilon$ are independent Gaussians: $\epsilon_i -\rE \epsilon_i \sim N(0,\sigma_i^2)$, where each $\sigma_i \leq \sigma$ ($i=1,\ldots,n$). Then Assumption~\ref{assump:noise} holds with
$a=\sigma$ and $b=\sqrt{2} \sigma$.
\end{proposition}

The next assumption handles the case that true target is not
exactly sparse. That is, we only assume that
$X \bar{\beta} \approx \rE \by$.
\begin{assumption}[Group approximation error condition]
\label{assump:approx}
There exist $\delta a, \delta b \geq 0$
such that for all group $j \in \{1,\ldots,m\}$:
the projection of error mean $\rE \epsilon$ to the $j$-th group
is bounded by:
\[
\|(X_{G_j}^\top X_{G_j})^{-0.5} X_{G_j}^\top \rE \epsilon\|_2/\sqrt{n} \leq
\sqrt{k_j} \delta a + \delta b .
\]
\end{assumption}

As mentioned earlier, we do not assume that the noise is zero-mean.
Hence $\rE \epsilon$ may not equal zero. 
In other words, 
this condition considers the situation that the true target is not exactly
sparse. It resembles algebraic noise in \cite{Zhang07-l1} but
takes the group structure into account. Similar to  \cite{Zhang07-l1},
we have the following result.
\begin{proposition} \label{prop:approx}
Consider a $(g,k)$ strongly group sparse coefficient vector
$\bb$ such that 
\[
\frac{1}{n}\|X \bb - \rE \by\|_2^2 \leq \Delta^2 ,
\]
and $a_0, b_0 \geq 0$. Then there exists $(g',k')$ strongly group sparse
$\bb'$ such that $ k' a_0^2 + g' b_0^2 \leq 2(k a_0^2 + g b_0^2)$, 
$\|X \bb' - \rE \by\|_2 \leq \|X \bb - \rE \by\|_2$,
$\Fr(\bb) \subset \Fr(\bb')$, 
and for all group $j$:
\[
\|(X_{G_j}^\top X_{G_j})^{-0.5} X_{G_j}^\top (X \bb'-\rE \by)\|_2/\sqrt{n} \leq
(a_0 \sqrt{k_j} + b_0) \Delta/\sqrt{k a_0^2 + b_0^2} .
\]
\end{proposition}

The proposition shows that if the approximation error of $\bb$
is $\Delta=\|X \bb - \rE \by\|_2/\sqrt{n}$, then we may
find an alternative target $\bb'$ with similar sparsity for which we can
take $\delta a = a_0 \Delta/\sqrt{k a_0^2 + b_0^2}$ and $\delta b= b_0 \Delta/\sqrt{k a_0^2 + b_0^2}$
in Assumption~\ref{assump:approx}. This means that
in Theorem~\ref{thm:group} below,
by choosing $a_0=a$ and $b_0=b \sqrt{\ln (m/\eta)}$, 
the contribution of the approximation error to the
reconstruction error $\|\hb-\bb\|_2$ is $O(\Delta)$.
Note that
this assumption does not show the benefit of group Lasso over standard Lasso.
Therefore in order to compare our results to that of the standard Lasso,
one may consider the simple situation where $\delta a = \delta b=0$.
That is, the target is exactly sparse.
The only reason to include
Assumption~\ref{assump:approx} is to illustrate that our analysis can handle
approximate sparsity.

The last assumption is a sparse eigenvalue condition, used in the
modern analysis of Lasso (e.g., \cite{BiRiTs07,Zhang07-l1}).
It is also closely
related to (and slightly weaker than) the RIP (restricted isometry property)
assumption \cite{CandTao05-rip} in the compressive sensing literature.
This assumption takes advantage of group structure, and can be
considered as (a weaker version of) group RIP. We introduce a definition
before stating the assumption.
\begin{definition}
For all $F \subset \{1,\ldots,p\}$, define
\begin{align*}
\rho_-(F)=&\inf \left\{\frac{1}{n}\|X \beta\|_2^2/\|\beta\|_2^2: \Fr(\beta) \subset F\right\} , \\
  \rho_+(F)=&\sup \left\{\frac{1}{n}\|X \beta\|_2^2/\|\beta\|_2^2: \Fr(\beta) \subset F\right\} .
\end{align*}
Moreover, for all $1 \leq s \leq p$, define
  \begin{align*}
    \rho_-(s) =& \inf \{\rho_-(G_S): S \subset \{1,\ldots,m\},
    |G_S| \leq s\} , \\
    \rho_+(s) =& \sup \{\rho_+(G_S): S \subset \{1,\ldots,m\} ,
    |G_S| \leq s\} .
  \end{align*}
\end{definition}

\begin{assumption}[Group sparse eigenvalue condition]
\label{assump:eigen}
There exist $s,c >0$ such that
\[
\frac{\rho_+(s)-\rho_-(2s)}{\rho_-(s)} \leq c .
\]
\end{assumption}

Assumption~\ref{assump:eigen} illustrates another advantage of group Lasso
over standard Lasso. Since we only consider eigenvalues for sub-matrices
consistent with the group structure $\{G_j\}$, the ratio
$\rho_+(s)/\rho_-(s)$ can be significantly smaller than the corresponding
ratio for Lasso (which considers all subsets of $\{1,\ldots,p\}$
up to size $s$).
For example, assume that all group sizes are identical
$k_1=\ldots = k_m=k_0$, and $s$ is a multiple of $k_0$.
For random projections used in compressive sensing applications,
only $n= O(s + (s/k_0) \ln m)$ projections are needed for Assumption~\ref{assump:eigen} to hold. In comparison, for standard Lasso, we need
$n= O(s \ln p)$ projections. The difference can be significant when $p$ and
$k_0$ are large.
More precisely, we have the following random projection sample complexity
bound for the group sparse eigenvalue condition.
Although we assume Gaussian random matrix in order to state explicit
constants, it is clear that
similar results hold for other sub-Gaussian random matrices.

\begin{proposition} [Group-RIP]\label{prop:rip}
Suppose that elements in $X$ are iid standard Gaussian random variables
$N(0,1)$.
For any $t>0$ and $\delta \in (0,1)$, let
\[
n \geq \frac{8}{\delta^2}
[ \ln 3 + t  + k \ln (1+ 8/\delta) + g \ln (e m/g) ] .
\]
Then with probability at least $1-e^{-t}$,
the random matrix $X \in \mathbb{R}^{n \times p}$
satisfies the following group-RIP inequality
for all $(g,k)$ strongly group-sparse vector $\bb \in \R^p$,
\begin{equation}
 (1 -\delta)\| \bb \|_{2} \leq
\frac{1}{\sqrt{n}}
\| X \bb \|_{2} \leq (1+\delta)\| \bb \|_{2}. \label{eq:Group-RIP}
\end{equation}

\end{proposition}

\section{Main Results}

Our main result is the following
signal recovery (2-norm parameter estimation error) bound for group Lasso.
\begin{theorem} \label{thm:group}
Suppose that Assumption~\ref{assump:noise}, Assumption~\ref{assump:approx},
and Assumption~\ref{assump:eigen} are valid.
Take $\lambda_j= (A \sqrt{k_j} + B)/\sqrt{n}$,
where both $A$ and $B$ can depend on data $\by$.
Given $\eta \in (0,1)$,
with probability larger than $1-\eta$,
if the following conditions hold:
\begin{itemize}
\item $A \geq 4 \max_j \rho_+(G_j)^{1/2} (a + \delta a \sqrt{n})$,
\item $B \geq 4 \max_j \rho_+(G_j)^{1/2} (b \sqrt{\ln (m/\eta)}+ \delta b \sqrt{n})$,
\item $\bar{\beta}$ is a $(g,k)$ strongly group-sparse coefficient vector,
\item $s \geq k+k_0$,
\item Let $\ell=s-(k-k_0)+1$, and $g_\ell= \min \{ |S| : |G_S| \geq \ell, S \subset \{1,\ldots,m\}\}$, we have
\[
c^2 \leq \frac{\ell A^2 + g_\ell B^2}{72(k A^2 +g B^2)} ,
\]
\end{itemize}
then
the solution of (\ref{eq:group-lasso}) satisfies:
\[
\|\hb-\bb\|_2 \leq \frac{\sqrt{4.5}}{\rho_-(s) \sqrt{n}} (1+0.25 c^{-1}) \sqrt{A^2 k + g B^2} .
\]
\end{theorem}
The first four conditions of the theorem are not critical, as they are
just definitions and choices for $\lambda_j$. The fifth assumption
is critical, which means that the group sparse eigenvalue condition has
to be satisfied with some $c$ that is not too large.
In order to satisfy the condition, $\ell$ should be chosen relatively large
as the right hand side is linear in $\ell$. However, this implies that
$s$ also grow linearly. It is possible to find $s$ so that the condition is
satisfied when $c^2$ in Assumption~\ref{assump:eigen} grows sub-linearly in $s$.
Consider the situation that $\delta a=\delta b=0$. If the conditions of
Theorem~\ref{thm:group} is satisfied, then
\[
\|\hb-\bb\|_2^2 = O((k + g \ln (m/\eta))/n) .
\]
In comparison, The Lasso estimator can only achieve the bound
\[
\|\hb_{L1}-\bb\|_2^2 = O((\|\bb\|_0 \ln (p/\eta))/n).
\]
If $k/\|\bb\|_0 \ll \ln (p/\eta)$ (which means that the group structure is useful) and
$g \ll \|\bb\|_0$, then the group Lasso is superior. This is consistent with
intuition. However, if $k \gg \|\bb\|_0 \ln (p/\eta)$, then group Lasso
is inferior. This happens when the signal is not strongly group sparse.

Theorem~\ref{thm:group} also suggests that if the group sizes are not even,
then group Lasso may not work well when
the signal is contained in small sized groups.
This is because in such case $g_\ell$ can be significantly smaller than
$g$ even with relatively large $\ell$, which means we have to choose a large
$s$ and small $c$, implying a poor bound.
This prediction is confirmed in Section~\ref{exp:uneven}
using simulated data.
Intuitively, group Lasso favors large sized groups because the 2-norm
regularization for large group size is weaker.
Adjusting regularization parameters $\lambda_j$ not only
fails to work in theory, but also impractical since it is unrealistic to
tune many parameters.
This unstable behavior with respect to uneven group size
may be regarded as another drawback of the group Lasso
formulation.

In the following, we present two simplifications of Theorem~\ref{thm:group}
that are easier to interpret. The first is the compressive sensing case,
which does not consider stochastic noise.
\begin{corollary}[Compressive sensing] \label{cor:cs}
  Suppose that Assumption~\ref{assump:noise} and Assumption~\ref{assump:approx}
  are valid with $a=b=\delta b=0$.
  Take $\lambda_j= 4 \sqrt{k_j} \max_j \rho_+(G_j)^{1/2} \delta a$.
  Let $\bar{\beta}$ be a $(k,g)$ strongly group-sparse signal,
  $\ell=k$, and $s=2k+k_0-1$.
  If $(\rho_+(s)-\rho_-(2s))/\rho_-(s) \leq 1/\sqrt{72}$, then
  the solution of (\ref{eq:group-lasso}) satisfies:
  \[
  \|\hb-\bb\|_2 \leq
  \frac{6\sqrt{2} + 18}{\rho_-(s)}
  \max_j \rho_+(G_j)^{1/2} \delta a \sqrt{k} .
  \]
\end{corollary}
If $\delta a=0$, then we can achieve exact recovery. Moreover,
Proposition~\ref{prop:approx} implies that we may choose 
a target with similar sparsity
such that $\delta a \sqrt{k} =O(\|X \bb - \rE \by\|_2 /\sqrt{n})$.
This implies a bound
\[
  \|\hb-\bb\|_2 =O(\| X \bb - \rE \by\|_2/\sqrt{n}) .
\]
If we have even sized groups, the number of samples $n$ required for
Corollary~\ref{cor:cs} to hold 
(that is, $(\rho_+(s)-\rho_-(2s))/\rho_-(s) \leq 1/\sqrt{72}$)
is $O(k+ g \ln (m/g))$, where $g=k/k_0$.
In comparison, although a similar result holds for Lasso, it 
requires sample size
of order $\|\bb\|_0 \ln (p/\|\bb\|_0)$. Again, group Lasso has a significant
advantage if $k/\|\bb\|_0 \ll \ln (p/\|\bb\|_0)$, $g \ll \|\bb\|_0$, and $p$ is large.

The following corollary is for even sized groups, and the result is
simpler to interpret. For standard Lasso, $B = O(\sqrt{\ln p})$, and for group
Lasso, $B=O(\sqrt{\ln m})$. The benefit of group Lasso is the division
of $B^2$ by $k_0$ in the bound, which is a significant improvement when
the dimensionality $p$ is large. 
The disadvantage of group Lasso is that the signal sparsity
$\|\bb\|_0$ is replaced by the group sparsity $k$.
This is not an artifact of our analysis, but rather a fundamental
drawback inherent to the group Lasso formulation. The effect is observable,
as shown in our simulation studies.
\begin{corollary}[Even group size]
 Suppose that Assumption~\ref{assump:noise} and Assumption~\ref{assump:approx}
are valid.
Assume also that all groups are of equal sizes:
$k_0=k_j$ for $j=1,\ldots,m$. Given $\eta \in (0,1)$,
let
\[
\lambda_j= (A \sqrt{k_0} + B)/\sqrt{n} ,
\]
where $A \geq 4 \max_j \rho_+(G_j)^{1/2} (a + \delta a \sqrt{n})$
and $B \geq 4 \max_j \rho_+(G_j)^{1/2} (b \sqrt{\ln (m/\eta)} + \delta b \sqrt{n})$.
Let $\bar{\beta}$ be a $(k,k/k_0)$ strongly group-sparse signal.
With probability larger than $1-\eta$,
if
\[
6 \sqrt{2} (\rho_+(k+\ell)-\rho_-(2k+2\ell)) /\rho_-(k+\ell) < \sqrt{\ell/k}
\]
for some $\ell >0$ that is a multiple of $k_0$,
then the solution of (\ref{eq:group-lasso}) satisfies:
\[
\|\hb-\bb\|_2 \leq \rho_-(k+\ell)^{-1} (\sqrt{4.5}+4.5 \ell/k)
\sqrt{A^2 + B^2/k_0}
\sqrt{k/n} .
\]
\end{corollary}

\section{Simulation Studies}

We want to verify our theory
by comparing group Lasso to Lasso on simulation data.
For quantitative evaluation, the recovery error is defined as
the relative
difference in 2-norm between the estimated sparse coefficient vector
$\beta_{est}$ and the ground-truth sparse coefficient
$\bb$: $\|\beta_{est}-\bb\|_{2}/\|\bb\|_{2}$.

The regularization parameter $\lambda$ in Lasso is chosen with
five-fold cross validation. In group Lasso, we
simply suppose the regularization parameter
$\lambda_{j}=(\lambda\sqrt{k_{j}})/\sqrt{n}$ for $j=1,2,...,m$.
The regularization parameter $\lambda$ is then chosen with 
five-fold cross validation. Here we set $B=0$
in the formula $\lambda_j=O(A\sqrt{k_j}+B)$. 
Since the relative performance of group Lasso versus standard
Lasso is similar with other values of $B$,
in order to avoid redundancy, we do not
include results with $B \neq 0$.

\subsection{Even group size}
In this set of experiments, the projection matrix $X$ is generated
by creating an $n \times p$ matrix with i.i.d. draws from a standard
Gaussian distribution $N(0, 1)$. For simplicity, the rows of $X$
are normalized to unit magnitude. Zero-mean Gaussian noise with
standard deviation $\sigma=0.01$ is added to the measurements. Our
task is to compare the recovery performance of Lasso and Group
Lasso for these $(g,k)$ strongly group sparse signals. 

\subsubsection{With correct group structure}
In this experiment,  we randomly generate $(g,k)$ strongly group
sparse coefficients with values $\pm 1$, where $p=512$, $k=64$ and
$g=16$. There are 128 groups with even group size of $k_0=4$. Here
the group structure coincides with the signal sparsity:
$k=\|\bb\|_0$.

Figure~\ref{fig:Exp_even_example_correct}
shows an instance of generated sparse coefficient vector
and the recovered results
by Lasso and group Lasso respectively when $n=3k=192$.
Since the sample size $n$ is only three times the signal sparsity
$k$, the standard Lasso does not achieve good recovery results, whereas the
group Lasso achieves near perfect recovery of the original signal.

Figure~\ref{fig:Exp_even_stat_correct} shows the effect of
sample size $n$, where we report the averaged recover error
over 100 random runs for each sample size.
Group Lasso is clearly superior in this case.
These results show that the
the group Lasso can achieve better recovery performance for
$(g,k)$ strongly group sparse signals with fewer measurements,
which is consistent with our theory.

\begin{figure}[htbp]
\centering
\includegraphics[width = 0.7 \columnwidth]{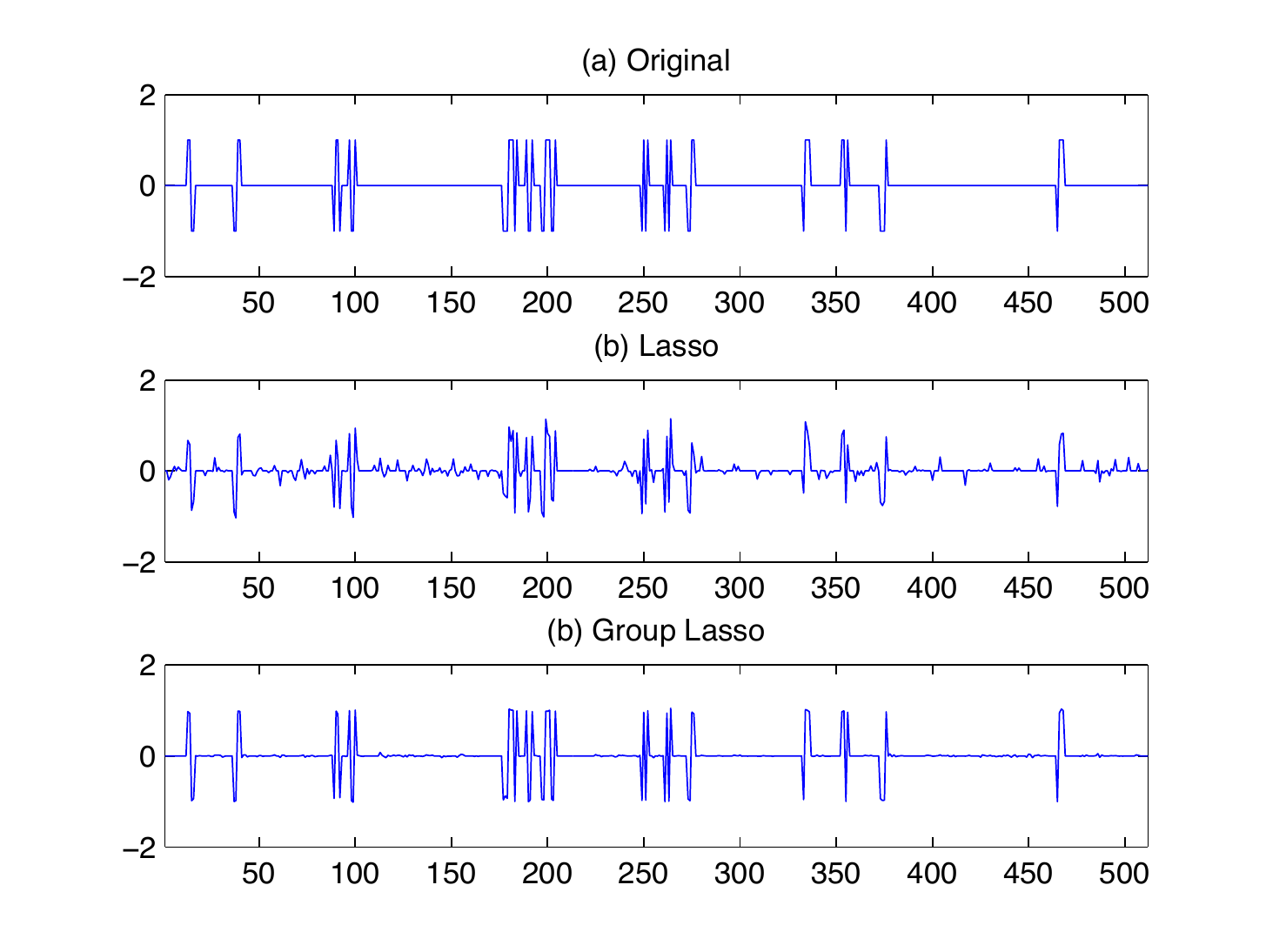}
\caption{Recovery results when the assumed group structure is correct. (a) Original data; (b) results with
Lasso (recovery error is 0.3444); (c) results with Group Lasso
(recovery error is 0.0419)} \label{fig:Exp_even_example_correct}
\end{figure}

\begin{figure}[htbp]
\centering
    \subfigure[]{\label{fig:Exp_even_stat_correct}
        \includegraphics[scale=0.46]{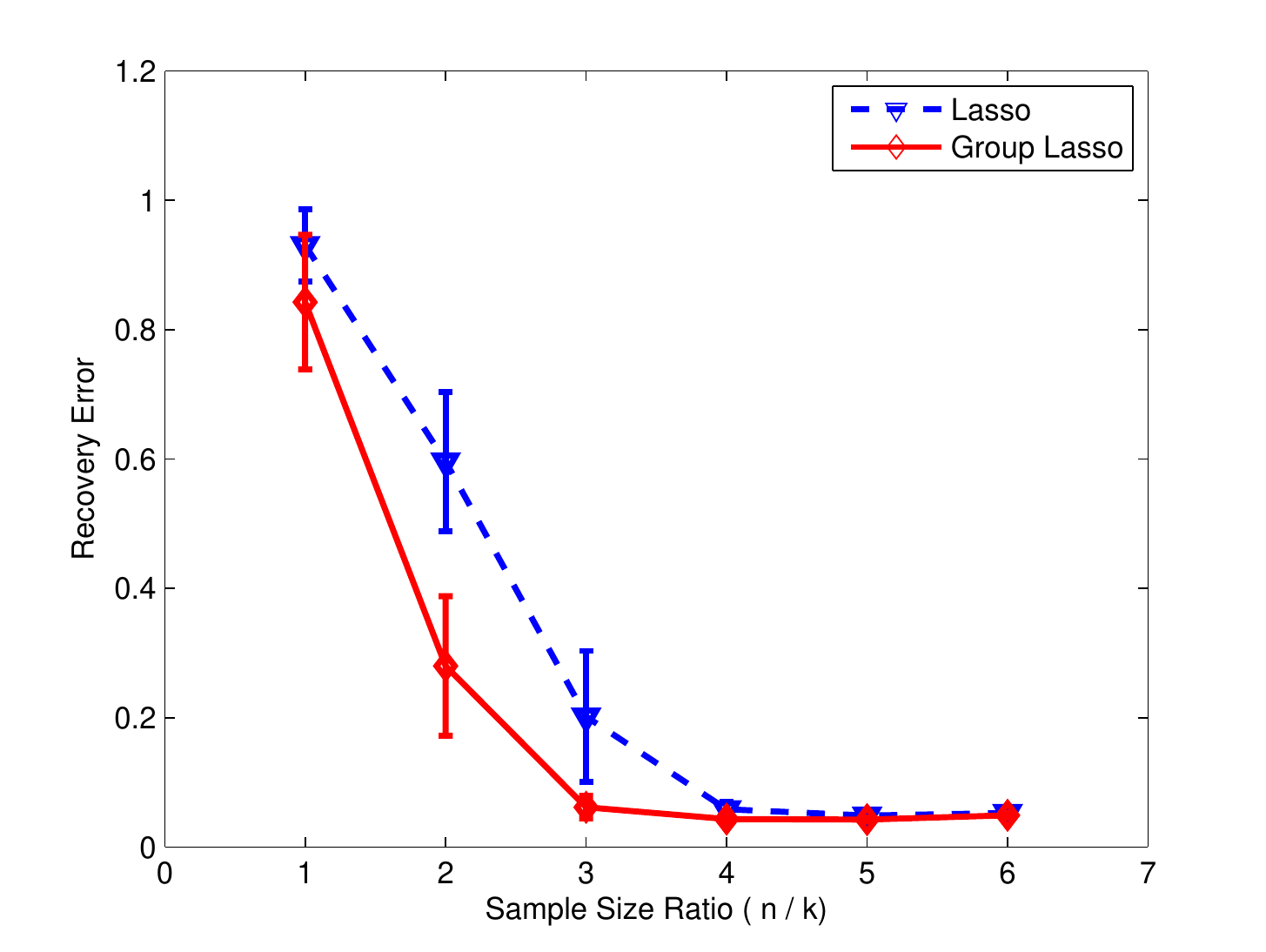}}
    \subfigure[]{\label{fig:Exp_even_stat_correct_g}
        \includegraphics[scale=0.46]{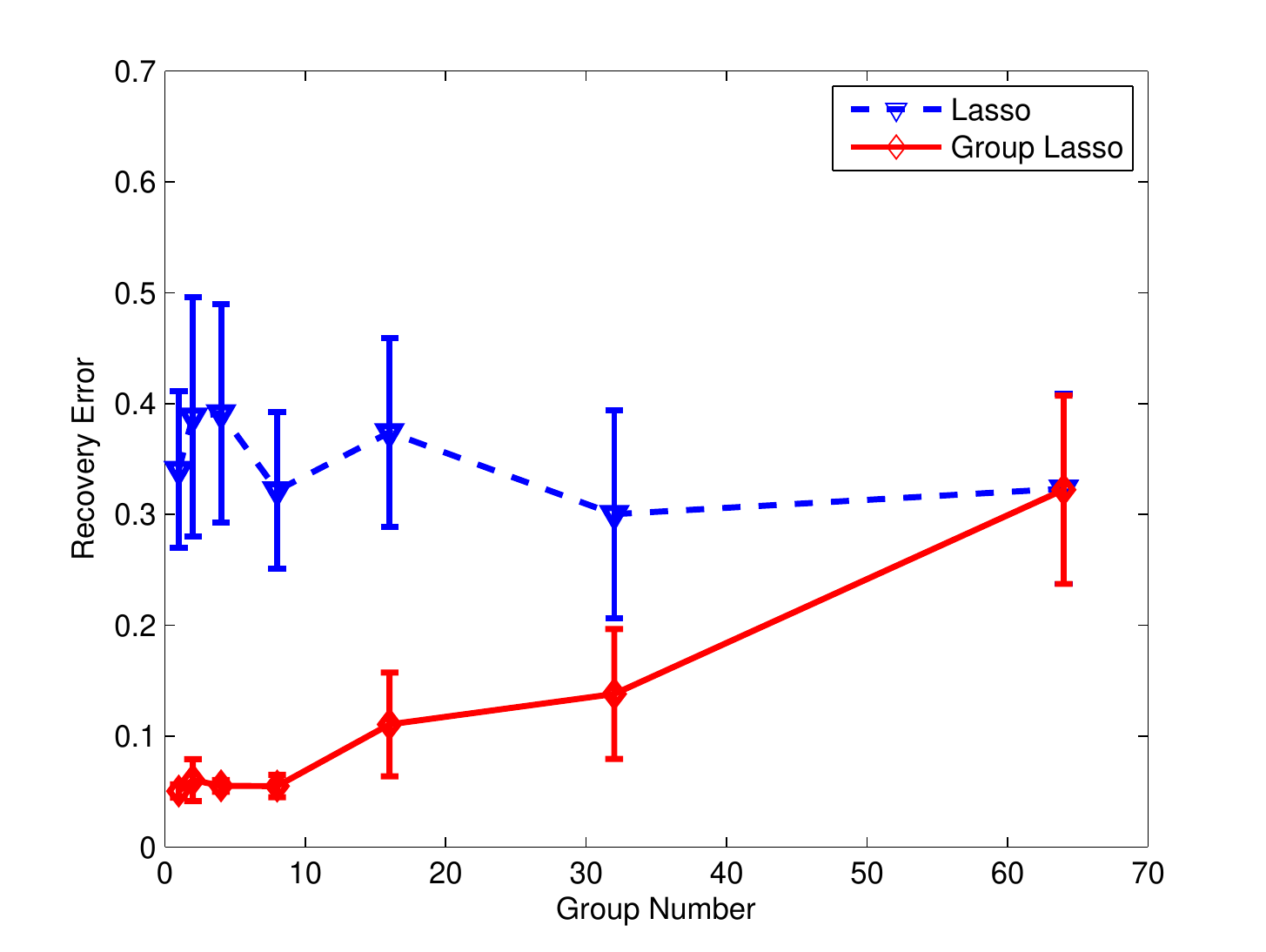}}
\caption{Recovery performance: (a) recovery error vs. sample size
ratio $n/k$; (b) recovery error vs. group number $g$}
\end{figure}

To study the effect of the group number $g$ (with $k$ fixed),
we set the sample size $n=160$ and then change
the group number while keeping other parameters unchanged.
Figure~\ref{fig:Exp_even_stat_correct_g} shows the recovery
performance of the two algorithms, averaged over 100 random runs
for each sample size.
As expected,
the recovery performance for Lasso is independent to the group number within
statistical error.
Moreover, the recovery results for group Lasso
are significantly better when the group number $g$ is much smaller than
the sparsity $k=64$.
When $g=k$, the group Lasso becomes identical to Lasso, which is
expected.
This shows that the recovery performance of group Lasso degrades
when $g/k$ increases, which confirms our theory.

\subsubsection{With incorrect group structure}
In this experiment, we assume that the known group structure is
not exactly the same as the sparsity of the signal (that is, $k >
\|\bb\|_0$). We randomly generate strongly group sparse
coefficients with values $\pm 1$, where $p=512$, $\|\bb\|_0=64$
and $g=16$. In the first experiment, we let $k = 4 \|\bb\|_0$, and
use $m=32$ groups with even group size of $k_0=16$.

Figure~\ref{fig:Exp_even_example_incorrect} shows one instance of
the generated sparse signal and the recovered results by Lasso and group Lasso
respectively when $n=3\|\bb\|_0=192$.
In this case, the standard Lasso obtains better recovery results than
the group Lasso.
Figure~\ref{fig:Exp_even_stat_correct} shows the effect of
sample size $n$, where we report the averaged recover error
over 100 random runs for each sample size.
The group Lasso recovery performance is clearly inferior to
that of the Lasso.
This shows that group Lasso fails when $k/\|\bb\|_0$ is relatively
large, which is consistent with our theory.

To study the effect of $k/\|\bb\|_0$ on the group Lasso
performance, we keep $\|\bb\|_0$ fixed, and simply vary the group
size as $k_0=1,2,4,8,16,32,64$ with $k/\|\bb\|_0=1,1,1,2,4,8,16$.
Figure~\ref{fig:Exp_even_stat_incorrect_gs} shows the performance
of the two algorithms with different group sizes $k_0$ in terms of
recovery error. It shows that the performance of group Lasso is
better when $k/\|\bb\|_0=1$. However, when $k/\|\bb\|_0
>1$, the performance of group Lasso deteriorates.

\begin{figure}[htbp]
\centering
\includegraphics[width = 0.7 \columnwidth]{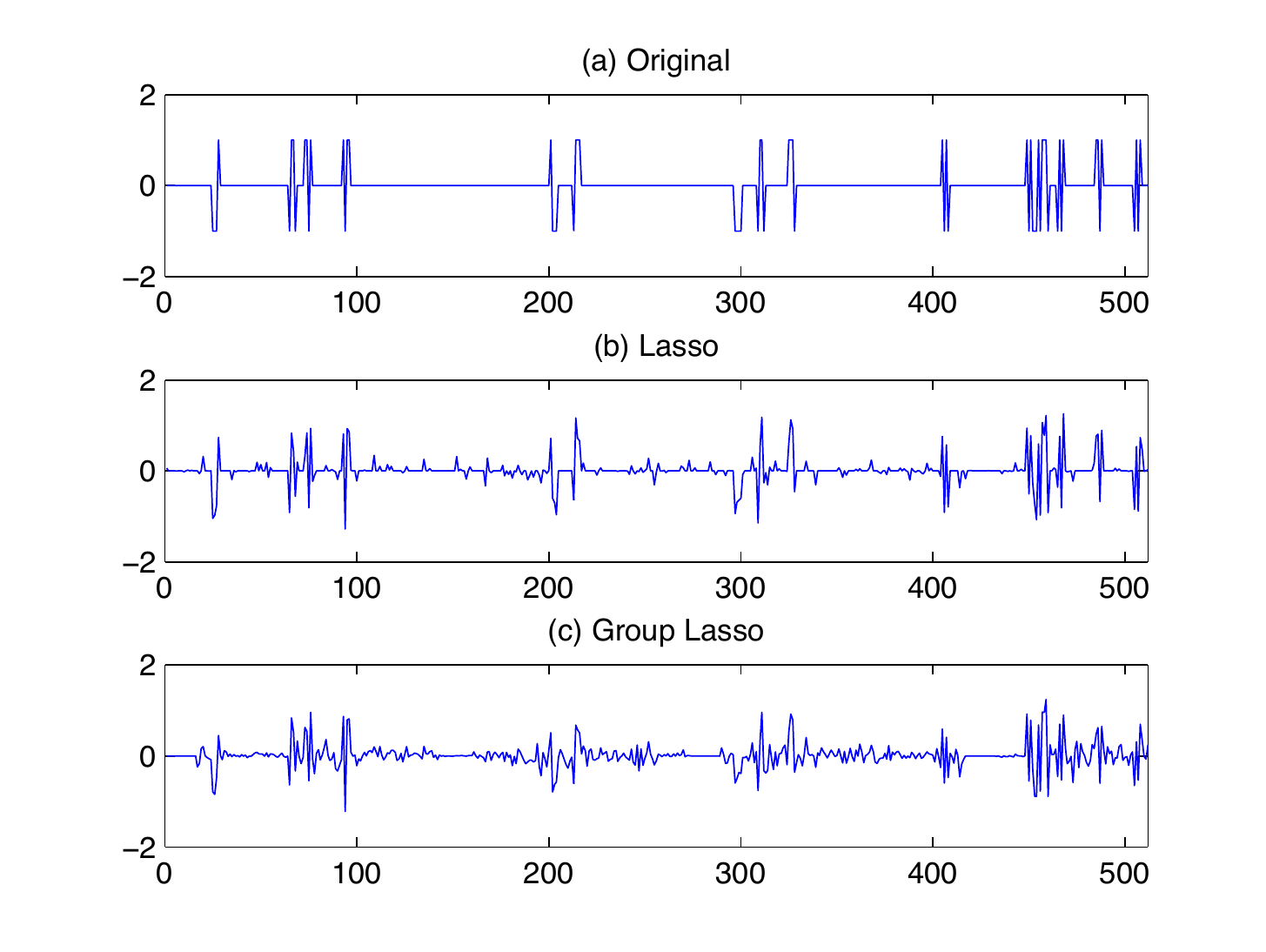}
\caption{Recovery results when the assumed group structure is
incorrect. (a) Original data; (b) results with Lasso (recovery
error is 0.3616); (c) results with Group Lasso (recovery error is
0.6688)} \label{fig:Exp_even_example_incorrect}
\end{figure}

\begin{figure}[htbp]
\centering
    \subfigure[]{\label{fig:Exp_even_stat_incorrect_m}
        \includegraphics[scale=0.46]{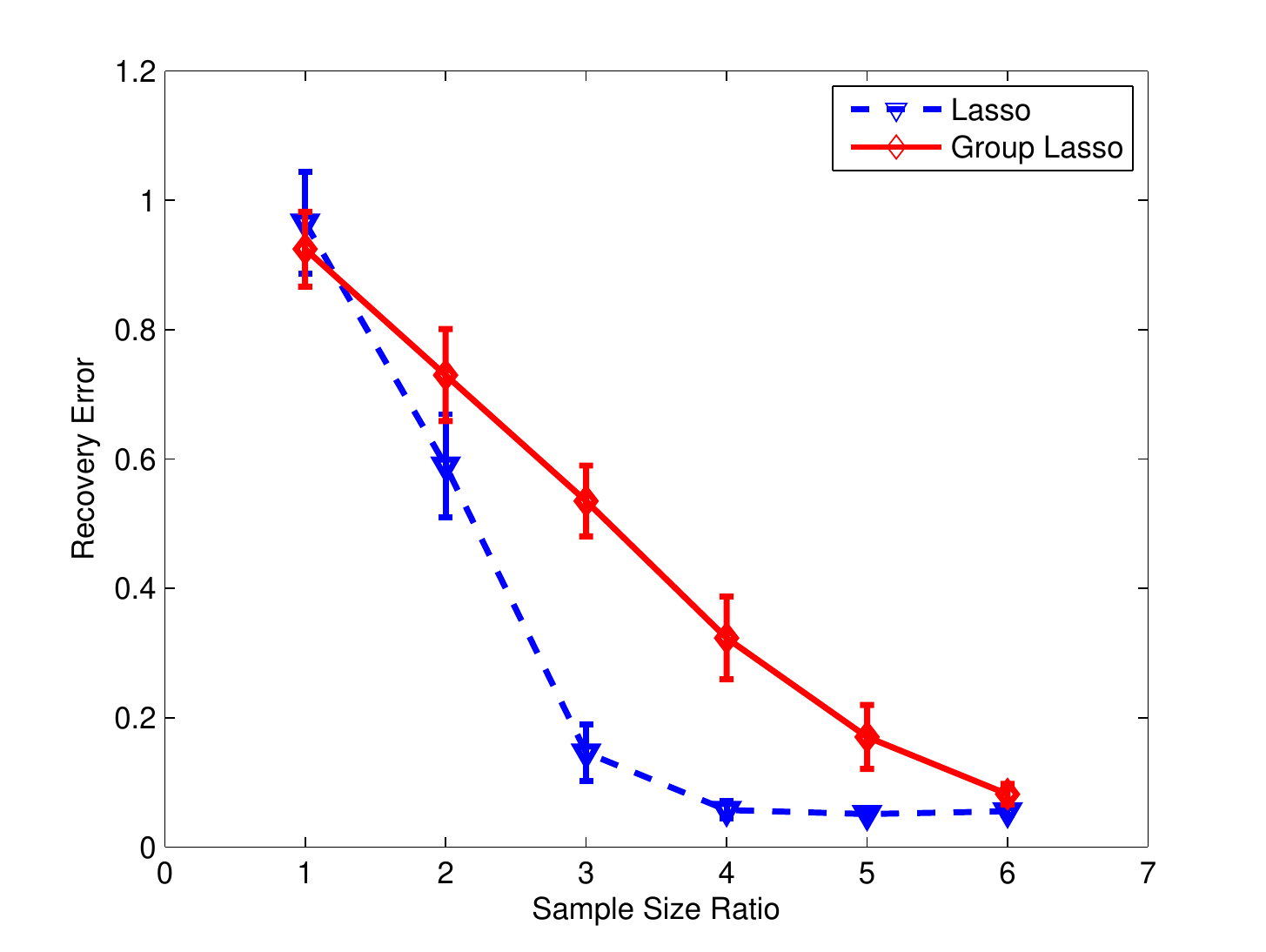}}
    \subfigure[]{\label{fig:Exp_even_stat_incorrect_gs}
        \includegraphics[scale=0.46]{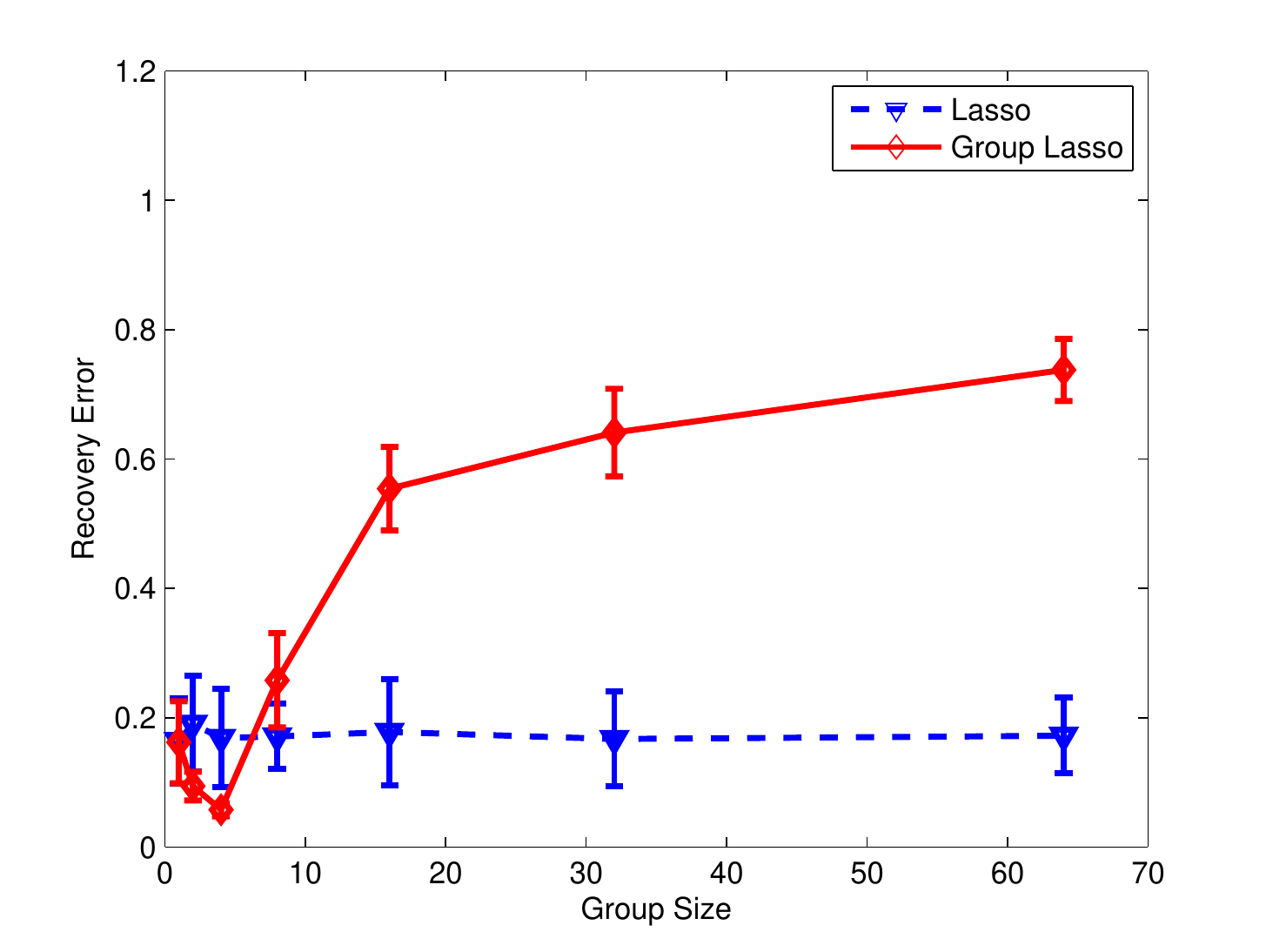}}
\caption{Recovery performance: (a) recovery error vs. sample size
ratio $n/k$; (b) recovery error vs. group size $k_0$}
\end{figure}

\subsection{Uneven group size}
\label{exp:uneven}
In this set of experiments, we randomly generate $(g,k)$ strongly
sparse coefficients with values $\pm 1$, where $p=512$, and $g=4$.
There are 64 uneven sized groups. The projection matrix
$X$ and noises are generated as in
the even group size case. Our task is to compare the recovery
performance of Lasso and Group Lasso for $(g,k)$ strongly
sparse signals with $\|\bb\|_0=k$.
To reduce the variance, we run each experiment 100 times and
report the average performance.

In the first experiment, the group sizes of $64$ groups are randomly generated
and the $g=4$ active groups are randomly extracted from these $64$ groups.
Figure~\ref{fig:Exp_uneven_stat_m_rand} shows the recovery
performance of Lasso and group Lasso with increasing sample size
(measurements) in terms of recovery error. Similar to the case of even group
size, the group Lasso obtains better recovery results than those
with Lasso. It shows that the group Lasso is superior when the
group sizes are randomly uneven.

\begin{figure}[htbp]
\centering
    \subfigure[]{\label{fig:Exp_uneven_stat_m_rand}
        \includegraphics[scale=0.46]{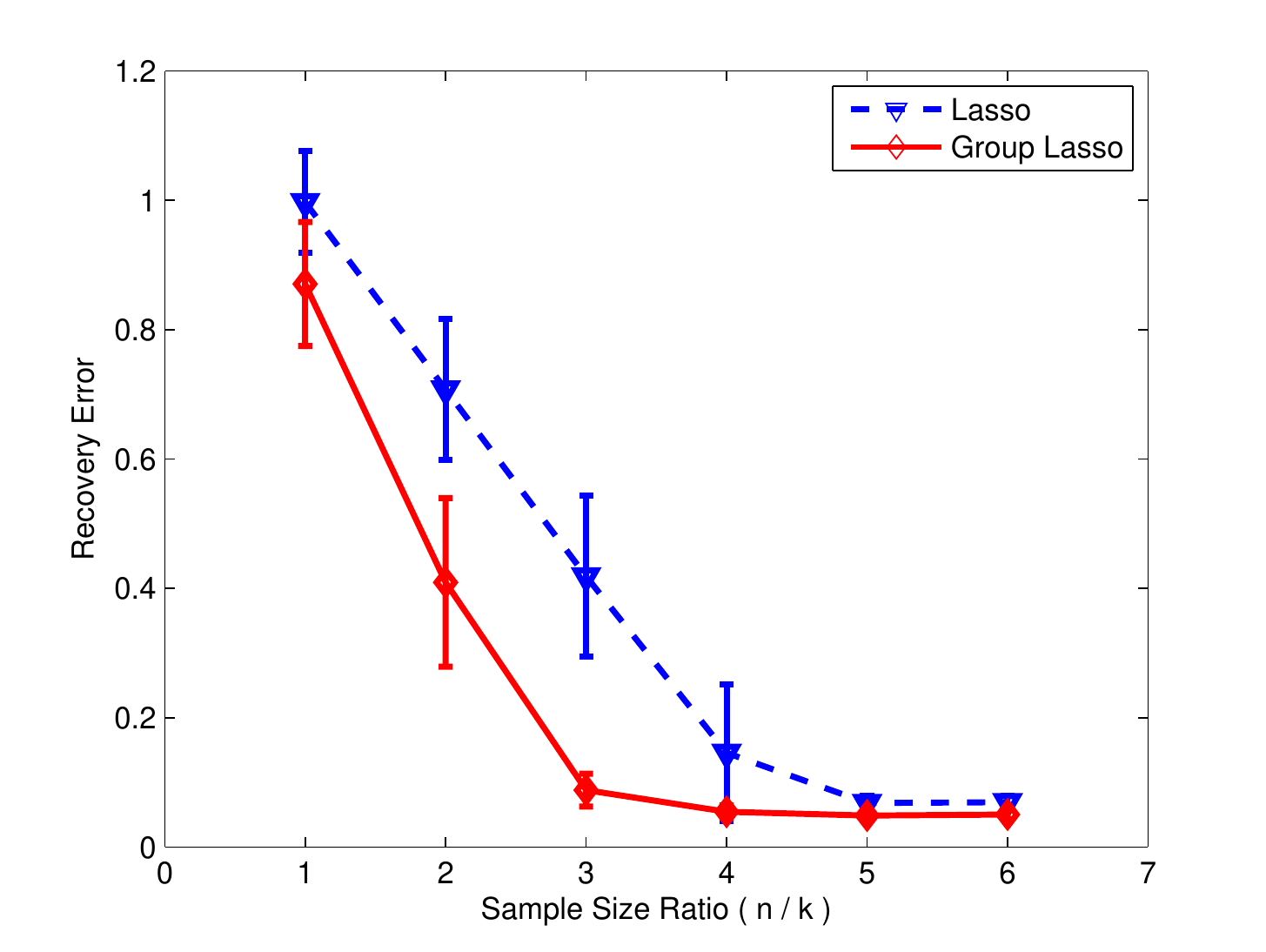}}
    \subfigure[]{\label{fig:Exp_uneven_stat_m_halfmixed}
        \includegraphics[scale=0.46]{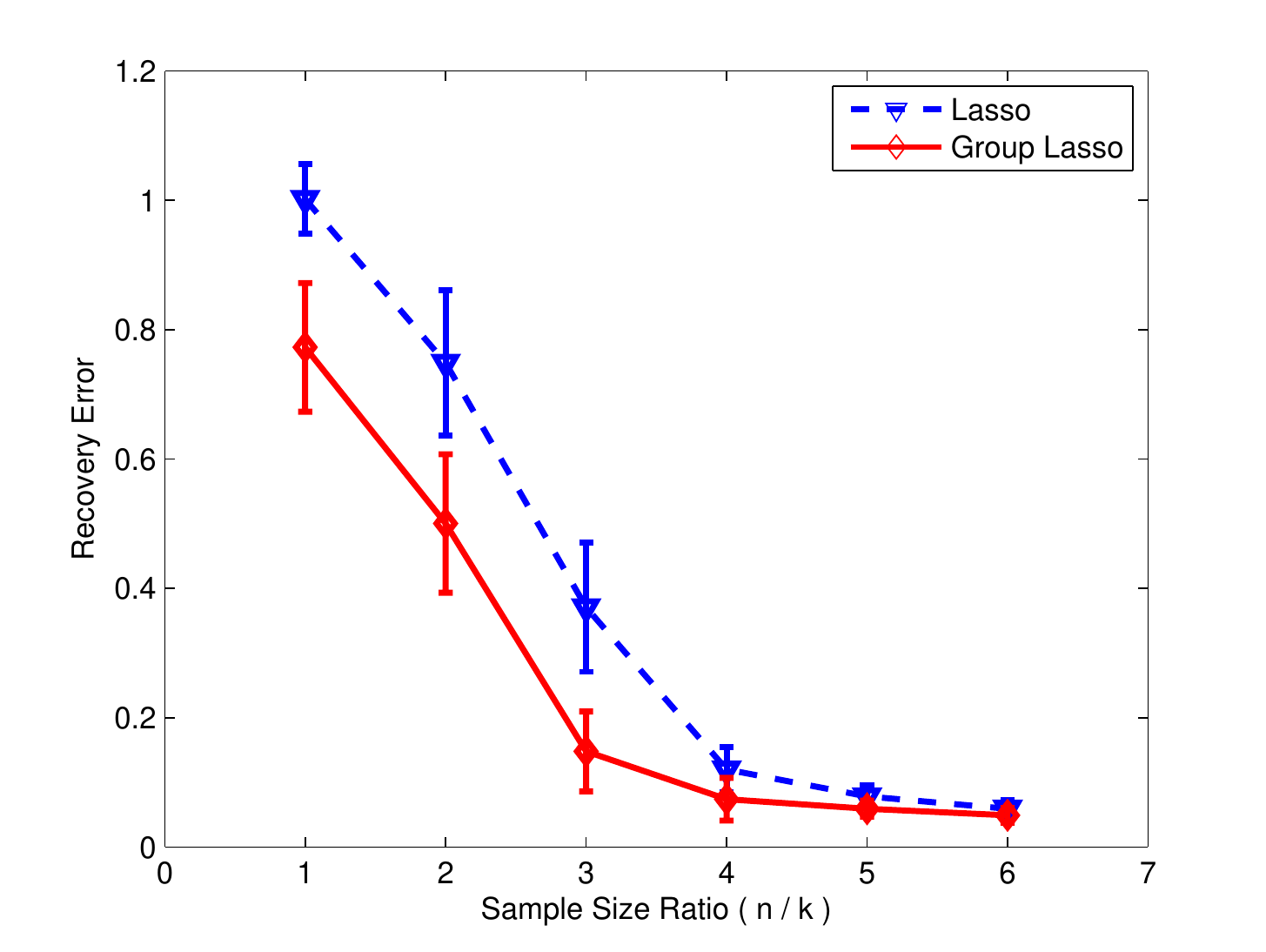}}
\caption{Recovery performance: (a) $g$ active groups have randomly
uneven group sizes; (b) half of $g$ active groups are single
element groups and another half of $g$ active groups have large
group size}
\end{figure}

As discussed after Theorem~\ref{thm:group}, because
group Lasso favors large sized groups,
if the signal is contained in small sized groups, then the
performance of group Lasso can be relatively poor.
In order to confirm this claim of Theorem~\ref{thm:group},
we consider the special case where $32$ groups have
large group sizes and each of the remaining $32$ groups has only
one element. First, we consider the case where half of $g=4$
active groups are extracted from the single element groups and
the other half of $g=4$ active groups are extracted from the groups
with large size. Figure~\ref{fig:Exp_uneven_stat_m_halfmixed}
shows the signal recovery performance of Lasso and group Lasso.
It is clear that the group Lasso performs better, but the
results are not as good as those of
Figure~\ref{fig:Exp_uneven_stat_m_rand}.

Moreover, Figure \ref{fig:Exp_uneven_stat_m_large} shows the recovery
performance of Lasso and group Lasso when
all of the $g=4$ active groups are extracted from large sized groups.
We observe that the relative performance of group Lasso improves.
Finally, Figure \ref{fig:Exp_uneven_stat_m_single} shows the recovery
performance of Lasso and group Lasso when
all of the $g=4$ active groups are extracted from single element groups.
It is obvious that the group Lasso is inferior to Lasso in this case.
This confirms the prediction of Theorem~\ref{thm:group} that suggests that
group Lasso favors large sized groups.

\begin{figure}[htbp]
\centering
    \subfigure[]{\label{fig:Exp_uneven_stat_m_large}
        \includegraphics[scale=0.46]{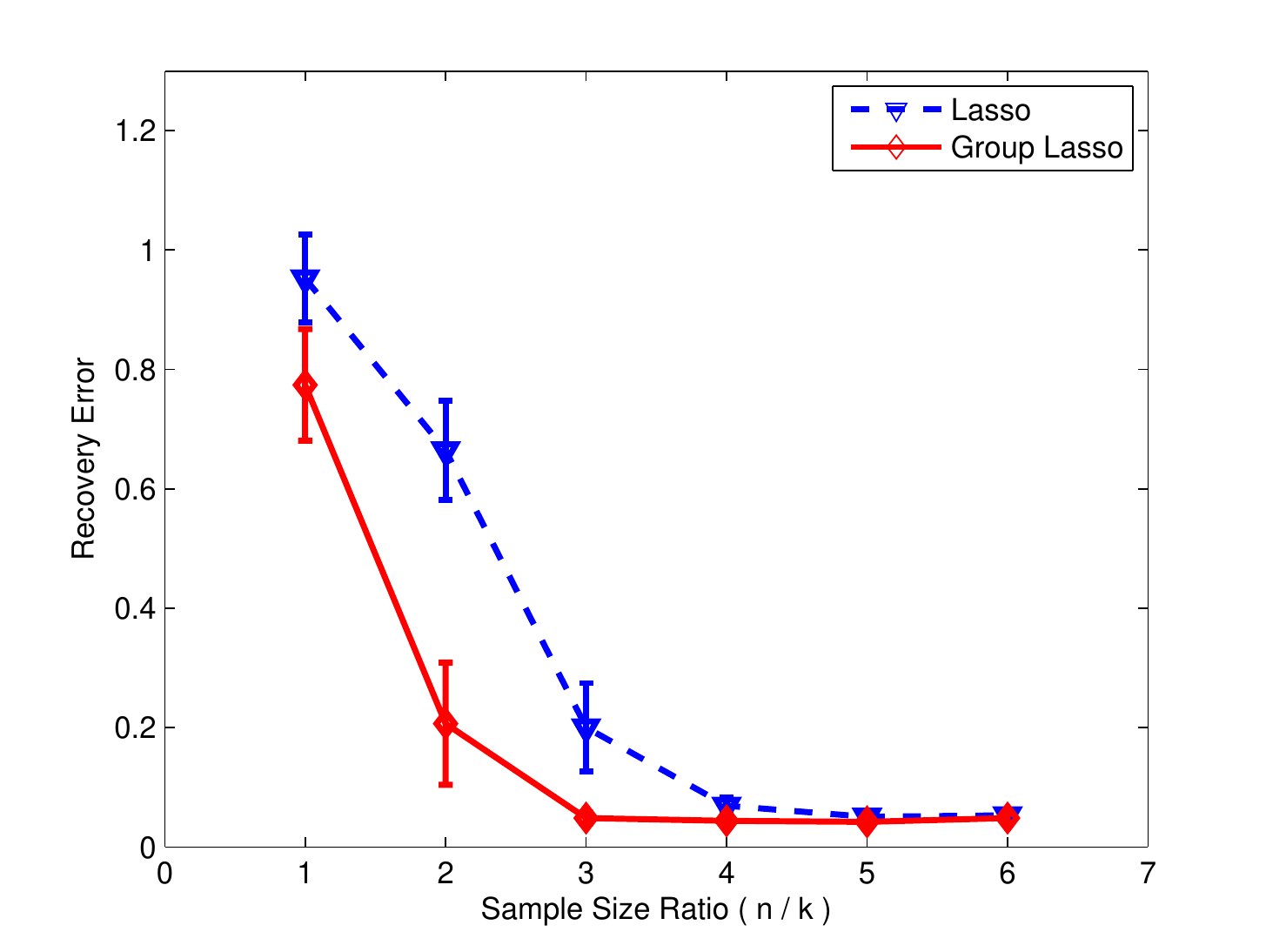}}
    \subfigure[]{\label{fig:Exp_uneven_stat_m_single}
        \includegraphics[scale=0.46]{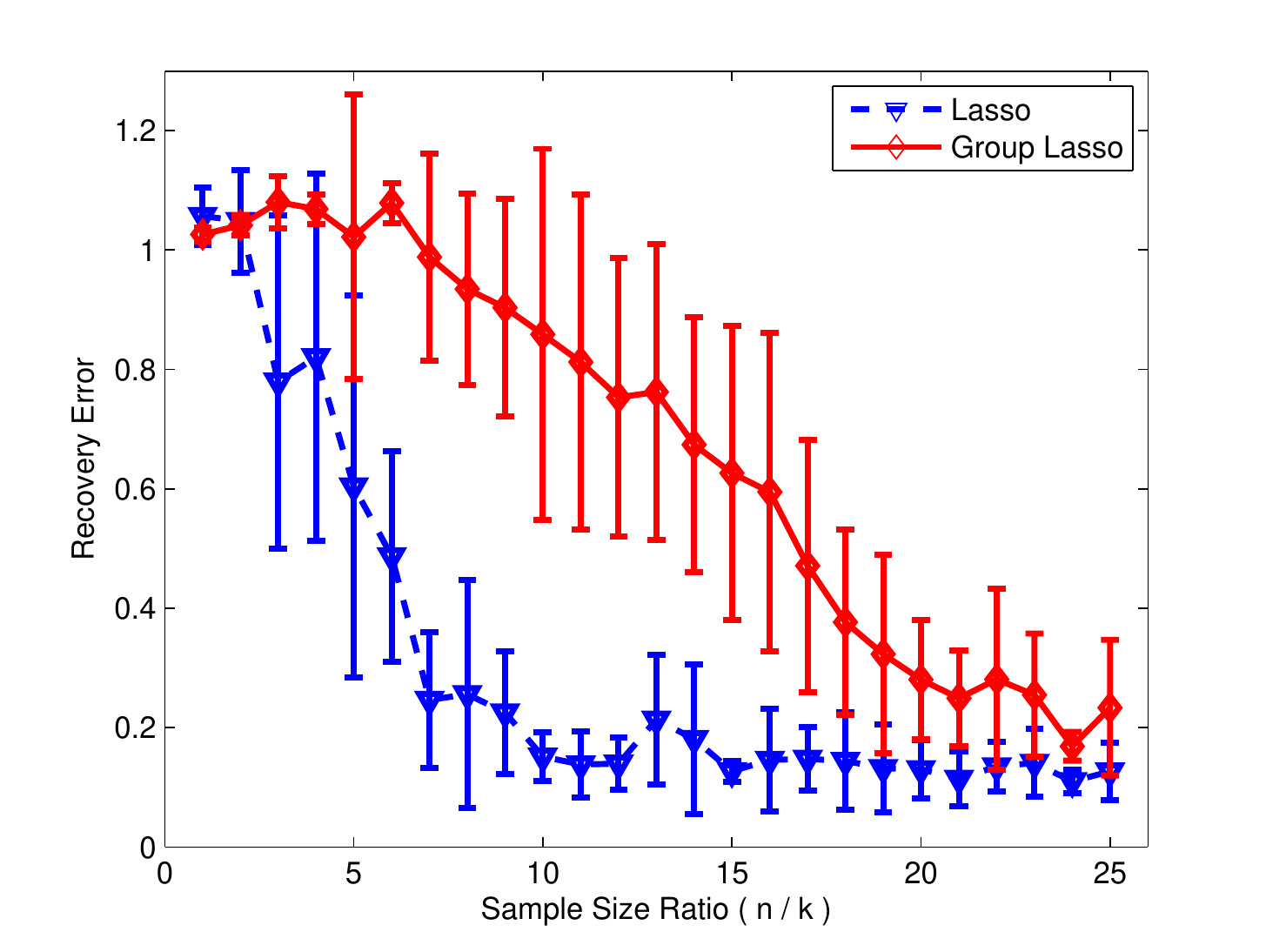}}
\caption{Recovery performance: (a) all $g$ active groups have large
group size; (b) all $g$ active groups are single
element groups}
\end{figure}

\section{Conclusion}

In this paper we introduced a concept called strong group sparsity that characterizes the signal recovery performance of group Lasso. In particular, we showed that group Lasso is superior to standard Lasso when the underlying signal is strongly group-sparse:
\begin{itemize}
\item Group Lasso is more robust to noise due to the stability associated with group structure.
\item Group Lasso requires a smaller sample size to satisfy the sparse eigenvalue condition required in the modern sparsity analysis.
\end{itemize}
However, group Lasso can be inferior if the signal is only weakly group-sparse, or covered by groups with small sizes. Moreover, group Lasso does not perform well with overlapping groups (which is not analyzed in this paper). Better learning algorithms are needed to overcome these limitations.

\bibliographystyle{plain}
\bibliography{group}

\appendix

\section{Proof of Proposition~\ref{prop:gaussian}}

Without loss of generality, we may assume $\sigma_i >0$ for all $i$
(otherwise, we can still let $\sigma_i>0$ and then just take the limit $\sigma_i \to 0$ for some $i$).

For notation
simplicity, we remove the subscript $j$ from the group index,
and consider group $G$ with $k$ variables.

Let $\Sigma$ be the diagonal matrix with $\sigma_i$ as its diagonal elements.
We can find an $n \times k$ matrix $Z=X_G (X_G^\top \Sigma X_G)^{-0.5} $, such that
$Z^\top \Sigma Z = I_{k \times k}$.
Let $\xi=Z^\top (\epsilon-\rE \epsilon) \in \R^k$. Since $\forall v \in \R^n$,
\[
\|(X_G^\top X_G)^{-0.5} X_G^\top v\|_2 = \|(Z^\top Z)^{-0.5} Z^\top v\|_2 ,
\]
we have
\begin{align*}
\frac{\|(X_{G}^\top X_{G})^{-0.5} X_{G}^\top (\epsilon- \rE \epsilon) \|_2^2}
{\xi^\top \xi}
\leq & \sup_{v \in \R^n} \frac{v^\top Z (Z^\top Z)^{-1} Z^\top v}{v^\top Z Z^\top v} \\
= & \sup_{u \in \R^k} \frac{u^\top (Z^\top Z)^{-1} u}{u^\top u}
=  \sup_{u \in \R^k} \frac{u^\top Z^\top \Sigma Z u}{u^\top (Z^\top Z) u} \\
\leq & \sup_{v \in \R^n} \frac{v^\top \Sigma v}{v^\top v} \leq \sigma^2 .
\end{align*}
Therefore, we only need to show that with probability at least $1-\eta$ for all $\eta \in (0,1)$:
\begin{equation}
\|\xi\|_2 \leq a \sqrt{k} + b \sqrt{-\ln \eta}
\label{eq:chi-dev}
\end{equation}
with $a=1$ and $b=\sqrt{2}$.

To prove this inequality,
we note that the condition $Z^\top \Sigma Z = I_{k \times k}$
means that the covariance matrix of $\xi$ is
$I_{k \times,k}$. Therefore the components of $\xi$ are $k$ iid Gaussians
$N(0,1)$, and the distribution of $\|\xi\|_2^2$ is $\chi^2$.
Many methods have been suggested to approximate the tail probability of
$\chi^2$ distribution. For example, a well-known approximation of
$\|\xi\|_2$ is the normal $N(\sqrt{k-0.5},0.5)$, which would imply
$a=b=1$ in (\ref{eq:chi-dev}).
In the following, we derive a slightly weaker tail probability bound using
direct integration of tail probability
for $\delta \geq \sqrt{k}$:
\begin{align*}
P(\|\xi\|_2^2 \geq \delta^2) =& \frac{1}{\Gamma(k/2)2^{k/2}} \int_{x \geq \delta^2} x^{k/2-1} e^{-x/2} d x \\
 =& \frac{2}{\Gamma(k/2)2^{k/2}} \int_{x \geq \delta} x^{k-1} e^{-x^2/2} d x \\
 =& \frac{2 \delta^{k-1}}{\Gamma(k/2)2^{k/2}} \int_{x \geq 0} e^{-(x+\delta)^2/2 +(k-1)\ln (1+x/\delta)} d x \\
 \leq & \frac{2 \delta^{k-1} e^{-\delta^2/2}}{\Gamma(k/2)2^{k/2}} \int_{x \geq 0} e^{-x^2/2 +x (-\delta+(k-1)/\delta)} d x \\
 \leq & \frac{\sqrt{2\pi} \delta^{k-1} e^{-\delta^2/2}}{\Gamma(k/2)2^{k/2}}
\leq \sqrt{0.5} (\delta/\sqrt{k})^{k-1} e^{-0.5\delta^2 + 0.5 k} \\
\leq& \sqrt{0.5} e^{-\delta^2/2 + 0.5k + (k-1) (\delta/\sqrt{k}-1)} \leq \sqrt{0.5} e^{- (\delta-\sqrt{k})^2/2} .
\end{align*}
This implies that (\ref{eq:chi-dev}) holds with $a=1$ and $b=\sqrt{2}$.

Note that in the above derivation,
we have used the following Sterling lower bound
for the Gamma function
\[
\Gamma(0.5 k) \geq \sqrt{2 \pi} (0.5k)^{0.5k-0.5} e^{-0.5 k} .
\]

\section{Proof of Proposition~\ref{prop:approx}}

We consider the following group-greedy procedure starting with
$\bb^{(0)}=\bb$, and form $(k^{(\ell)},g^{(\ell)})$ strongly group sparse
$\bb^{(\ell)}$  as follows for $\ell=1,2,\ldots$
\begin{itemize}
\item let $r^{(\ell-1)}= X \bb^{(\ell-1)} - \rE \by$,
\item let $j^{(\ell)}=\arg\max_j [\|(X_{G_j}^\top X_{G_j})^{-0.5} X_{G_j}^\top r^{(\ell-1)}\|_2/\sqrt{k_j a_0^2 + b_0^2}]$,
\item let $\bb^{(\ell)}=\bb^{(\ell-1)}$; and then reset its coefficients
in group $G_j$ as
$\bb^{(\ell)}_{G_j}= \bb^{(\ell)}_{G_j} - (X_{G_j}^\top X_{G_j})^{-1} X_{G_j}^\top r^{(\ell-1)}$, where $j=j^{(\ell)}$.
\end{itemize}
It is not difficult to check that
\[
\| r^{(\ell-1)}\|_2^2 - \| r^{(\ell)}\|_2^2
= \|(X_{G_j}^\top X_{G_j})^{-0.5} X_{G_j}^\top r^{(\ell-1)}\|_2^2 ,
\]
$k^{(\ell)}- k^{(\ell-1)} \leq k_j$, 
$g^{(\ell)}- g^{(\ell-1)} \leq 1$, 
with $j=j^{(\ell)}$.
Therefore if for all $0 \leq \ell \leq t$, we have
\[
\arg\max_j \left[\|(X_{G_j}^\top X_{G_j})^{-0.5} X_{G_j}^\top r^{(\ell)}\|_2/\sqrt{k_j a_0^2 + b_0^2}\right] \geq \sqrt{n} \Delta/\sqrt{k a_0^2 + b_0^2} ,
\]
then by summing over $\ell=1,\ldots, t, t+1$, we obtain
\begin{align*}
n \Delta^2=&\| r^{(0)}\|_2^2 
\geq \sum_{\ell=1}^{t+1} [\|r^{(\ell-1)}\|_2^2- \| r^{(\ell)}\|_2^2] \\
\geq& n \sum_{\ell=1}^{t+1} [(k^{(\ell)}-k^{(\ell-1)}) a_0^2 + (g^{(\ell)}-g^{(\ell-1)}) b_0^2]
\Delta^2/(k a_0^2 + b_0^2) \\
\geq& n [(k^{(t+1)}-k) a_0^2 + (g^{(t+1)}-g) b_0^2]
\Delta^2/(k a_0^2 + b_0^2) .
\end{align*}
This implies that
\[
k^{(t+1)} a_0^2 + g^{(t+1)} b_0^2 \leq 2 (k a_0^2 + g b_0^2) .
\]
Therefore if we let $t$ be the first time 
$k^{(t+1)} a_0^2 + g^{(t+1)} b_0^2 > 2 (k a_0^2 + g b_0^2)$, then there exists
$\ell \leq t$, such that $\bb'=\beta^{(\ell)}$ satisfies the requirement.

\section{Proof of Proposition~\ref{prop:rip}}

The following lemma is taken from \cite{Pisier89}. Since the proof is simple,
it is included for completeness.
\begin{lemma}
  Consider the unit sphere $S^{k-1}=\{x: \|x\|_2 = 1\}$ in
  $\mathbb{R}^{k}$ ($k \geq 1$).
  Given any $\varepsilon > 0$,
  there exists an $\varepsilon$-cover  $Q \subset S^{k-1}$
  such that
  $\min_{q \in Q}\|x-q\|_{2}\leq \varepsilon$
  for all  $\|x\|_2 = 1$, with
  $|Q| \leq (1+2/\varepsilon)^{k}$.
  \label{lem:B1}
\end{lemma}
\begin{proof}
Let $B^{k}=\{x: \|x\|_2 \leq 1\}$ be the unit ball in $\mathbb{R}^{k}$.
Let $Q = \{q_{i}\}_{i=1,\ldots, |Q|} \subset S^{k-1}$ be a maximal subset such that
$\|q_{i}-q_{j}\|_{2}>\varepsilon$ for all $i \neq j$.
By maximality, $Q$ is an $\varepsilon$-cover of $S^{k-1}$.
Since the balls
$q_{i}+(\varepsilon/2)B^{k}$ are disjoint and belong to
$(1+\varepsilon/2)B^{k}$, we have
\[
\sum_{i\leq |Q|} vol(q_{i}+(\varepsilon/2)B^{k})\leq
vol((1+\varepsilon/2)B^{k}) .
\]
Therefore,
\[
|Q|(\varepsilon/2)^{k}vol(B^{k}) \leq
(1+\varepsilon/2)^{k}vol(B^{k}),
\]
which implies that $|Q| \leq (1+2/\varepsilon)^{k}$.
\end{proof}

The following concentration result for $\chi^2$ distribution is similar
to Proposition~\ref{prop:gaussian}. This is where the Gaussian assumption is used in the proof. A similar result holds for sub-Gaussian random variables.
\begin{lemma}
 Let $\xi \in \R^n$ be a vector of $n$ iid standard Gaussian variables:
 $\xi_i \sim N(0,1)$. Then $\forall \epsilon \geq 0$:
 \[
 \pr \left[ |\|\xi\|_2 - \sqrt{n}| \geq \epsilon \right]
 \leq 3 e^{ - \epsilon^2/2} .
 \]
 \label{lem:B2}
\end{lemma}
\begin{proof}
  Proposition~\ref{prop:gaussian} implies that
  \[
  \pr \left[ \|\xi\|_2 - \sqrt{n} \geq \epsilon \right]
  \leq  \sqrt{0.5} e^{- \epsilon^2/2} .
  \]
  Using identical derivation in the proof of  Proposition~\ref{prop:gaussian},
  and let $\delta=\sqrt{n}-\epsilon$ and $k=n$, we obtain:
  \begin{align*}
  \pr \left[ \|\xi\|_2 - \sqrt{n} \leq -\epsilon \right]
 \leq & \frac{2 \delta^{k-1} e^{-\delta^2/2}}{\Gamma(k/2)2^{k/2}} \int_{x \leq 0} e^{-x^2/2 +x (-\delta+(k-1)/\delta)} d x \\
 \leq & \frac{2 \delta^{k-1} e^{-\delta^2/2}}{\Gamma(k/2)2^{k/2}} \int_{x \leq 0} e^{-x^2/2 -x} d x \\
 \leq & 3 \times \frac{\sqrt{2\pi} \delta^{k-1} e^{-\delta^2/2}}{\Gamma(k/2)2^{k/2}} 
 \leq 3 \times \sqrt{0.5} e^{-\epsilon^2/2} .
\end{align*}
Combining the above two inequalities, we obtain the desired bound.
\end{proof}

The derivation of the following estimate employs a standard proof 
technique (for example, see \cite{RaScVa08}).
\begin{lemma}
  Suppose $X$ is generated according to Proposition~\ref{prop:rip}.
  For any fixed set $S \subset \{1,\ldots,p\}$ with $|S|=k$ and $0<\delta<1$,
  we have with probability exceeding
  $1-3(1+8/\delta)^{k}e^{- n \delta^2/8}$:
  \begin{equation}
  (1 -\delta)\|\beta\|_{2}\leq \frac{1}{\sqrt{n}} \|X_S \beta\|_{2}\leq(1+\delta)\|\beta\|_{2} \label{eq:Phi-norm-bound-all}
  \end{equation}
  for all $\beta \in \mathbb{R}^k$.
  \label{lem:B3}
\end{lemma}
\begin{proof}
  It is enough to prove the conclusion
  in the case of $\|\beta\|_{2}=1$.
  According to Lemma \ref{lem:B1}, given $\epsilon_1 >0$,
  there exists a finite set $Q=\{q_i\}$ with
  $|Q|\leq (1+2/\epsilon_{1})^{k}$
  such that $\|q_i\|_2=1$ for all $i$, and
  $\min_{i}\|\beta-q_i\|_{2}\leq \epsilon_1$
  for all $\|\beta\|_2=1$.

  For each $i$, Since elements of $\xi=X_S q_i$ are iid Gaussians
  $N(0,1)$, Lemma~\ref{lem:B2} implies that $\forall \epsilon_2>0$:
  \[
  \pr \left[| \|X_S q_i\|_{2}- \sqrt{n}\|q_i \|_{2}| \geq \sqrt{n} \epsilon_2 \right]
  \leq 3 e^{- n \epsilon_2^2/2} .
  \]
  Taking union bound for all $q_i \in Q$, we obtain
  with probability exceeding
  $1-3(1+2/\epsilon_{1})^{k}e^{ - n\epsilon_2^2/2}$:
  for all $q_i \in Q$,
  \[
  (1 - \epsilon_{2}) \leq \frac{1}{\sqrt{n}}\|X_S  q_i\|_{2}\leq(1+\epsilon_{2}) .
  \]
  Now, we define $\rho$ as the smallest nonnegative number such that
  \begin{equation}
   \frac{1}{\sqrt{n}} \|X_S \beta \|_{2}\leq (1+\rho) \label{eq:Phi-norm-bound}
  \end{equation}
  for all $\beta \in \mathbb{R}^k$ with $\|\beta\|_2=1$.
  Since for all $\|\beta\|_2=1$, we can find $q_i \in Q$ such that
  $\|\beta-q_i\|_2 \leq \epsilon_1$,  we have
  \[
  \|X_S \beta\|_{2}\leq \|X_S q_i\|_{2}+\|X_S (\beta-q_i)\|_{2} \leq
  \sqrt{n} (1 + \epsilon_2 + (1+\rho) \epsilon_1) ,
  \]
  where we used (\ref{eq:Phi-norm-bound}) in the derivation.
  Since $\rho$ is the smallest non-negative constant for which
  (\ref{eq:Phi-norm-bound}) holds, we have
  \[
  \sqrt{n} (1+\rho) \leq  \sqrt{n} (1 + \epsilon_2 + (1+\rho) \epsilon_1) ,
  \]
  which implies that
  \[
  \rho \leq (\epsilon_1+\epsilon_2)/(1-\epsilon_1) .
  \]
  Now we choose $\epsilon_{1}=\delta/4$ and $\epsilon_2=\delta/2$.
  Since $0<\delta<1$, it is easy to see that $\rho \leq
  \delta$. This proves the upper bound.
  For the lower bound, we note that for all $\|\beta\|_2=1$ with
  $\|\beta-q_i\|_2 \leq \epsilon_1$,  we have
  \[
  \|X_S \beta\|_{2}\geq \|X_S q_i\|_{2} -\|X_S (\beta-q_i)\|_{2} \geq
  \sqrt{n} (1 - \epsilon_2 - (1+\rho) \epsilon_1) ,
  \]
  which leads to the desired result.
\end{proof}

\subsection*{Proof of Proposition~\ref{prop:rip}}

For each subset $S \subset \{1,\ldots,m\}$ of groups with $|S| \leq g$ and
$|G_S|\leq k$, we know from \ref{lem:B3} that
for all $\beta$ such that $\Fr(\beta) \subset G_S$:
\[
(1 -\delta)\|\beta\|_{2}\leq\frac{1}{\sqrt{n}}\|X \beta\|_{2}\leq(1+\delta)\|\beta\|_{2}
\]
with probability exceeding $1-3(1+8/\delta)^{k}e^{ - n \delta^2/8}$.

Since the number of such groups $S$ can be no more than
$C_m^g \leq (em/g)^g$, by taking the union bound, we know that
the group RIP in Equation (\ref{eq:Group-RIP})
fails with probability less than
\[
3(em/g)^{g} (1+8/\delta)^k e^{-n \delta^2/8} \leq
e^{-t} .
\]

\section{Technical Lemmas}
The following lemmas are adapted from \cite{Zhang07-l1}
to handle group sparsity structure.
Similar techniques can be found in \cite{BiRiTs07}.
The first lemma is in  \cite{Zhang07-l1}. The proof is included for completeness.
\begin{lemma}
  Let $A=X^\top X/n$, and let
  $I$ and $J$ be non-overlapping indices in $\{1,\ldots,p\}$. We have
  \[
  \|A_{I,J}\|_2 \leq \sqrt{(\rho_+(I)-\rho_-(I \cup J)) (\rho_+(J)-\rho_-(I\cup J))} ,
  \]
  where the matrix 2-norm is defined as $\|A_{I,J}\|_2=\sup_{\|u\|_2=\|v\|_2=1} |u^\top A_{I,J} v|$.
  \label{lem:A2}
\end{lemma}
\begin{proof}
  Consider $v \in \R^p$ with $v_I \in \R^{|I|}$ and $v_J \in \R^{|J|}$:
  positive semi-definiteness implies that
  \begin{align*}
    & \rho_+(I) \|v_I\|_2^2 +  2 t v_I^\top A_{I,J} v_J
  + t^2 \rho_+(J) \|v_J\|_2^2 \\
  \geq& v_I^\top A_{I,I} v_I +  2 t v_I^\top A_{I,J} v_J + t^2 v_J^\top A_{J,J} v_J \\
  \geq &
  \rho_-(I \cup J) (\|v_I\|_2^2 + t^2 \|v_J\|_2^2)
\end{align*}
for all $t$. This implies that
  \[
  |v_I^\top A_{I,J} v_J| \leq
  \sqrt{(\rho_+(I)-\rho_-(I \cup J)) (\rho_+(J)-\rho_-(I\cup J))}
  \|v_I\|_2 \|v_J\|_2 ,
  \]
  which leads to the desired result.
\end{proof}

The next lemma uses the previous result to
control the contribution of the non-signal part
$G^c$ of an error vector $u$ 
to the product $u_G^\top A_{G,G^c} u_{G^c}$.
\begin{lemma}
  Given $u \in \R^p$ and $S \subset \{1,\ldots,m\}$.
  Consider $\ell \geq 1$ and define
  \[
  \lambda_-^2 = \min \left\{ \sum_{j \in S'} \lambda_j^2 : |G_{S'}| \geq \ell\right\} .
  \]
  Let $S_0 \subset \{1,\ldots,m\}-S$ contain
  indices $j$ of largest values of $\|u_{G_j}\|_2/\lambda_j$
  ($j \notin S$),
  and satisfies the condition $\ell \leq |G_{S_0}| < \ell+k_0$.
  Let $G=G_{S} \cup G_{S_0}$.
  Then
  \[
  \sqrt{\sum_{j \notin S \cup S_0} \|u_{G_j}\|_2^2}
  \leq  (2\lambda_-)^{-1} \sum_{j \notin S} \lambda_j \|u_{G_j}\|_2
  \]
  and
  \[
  \frac{1}{n} \left|\sum_{j \notin S \cup S_0} u_G^\top X_G^\top X_{G_j} u_{G_j}\right|
  \leq \lambda_-^{-1} \tilde{\rho}_+ \|u_G\|_2
  \sum_{j \notin S} \lambda_j \|u_{G_j}\|_2 ,
  \]
where $\tilde{\rho}_+ =\sqrt{(\rho_+(G)-\rho_-(|G|+\ell+k_0-1))
(\rho_+(\ell+k_0-1)- \rho_-(|G|+\ell+k_0-1))}$.
\label{lem:rip}
\end{lemma}
\begin{proof}
  Without loss of generality, we assume that
  $S=\{1,\ldots,g\}$, and we assume that $j>g$ is in descending order
  of $\|u_{G_j}\|_2/\lambda_j$.
  Let $S_0, S_1, \ldots$ be the first, second, etc, consecutive blocks of $j >g$,
  such that $\ell \leq |G_{S_k}| < \ell+k_0$ (except for the last $S_k$).
  If we let $G^k=G_{S_k}$, then:
  \begin{align*}
  \sum_{j \notin S \cup S_0} \|u_{G_j}\|_2^2 \leq&
  \left[\sum_{j \notin S \cup S_0} \lambda_j \|u_{G_j}\|_2\right]
  \left[\max_{j \notin S \cup S_0} \|u_{G_j}\|_2/\lambda_j\right] \\
  \leq&
  \left[\sum_{j \notin S \cup S_0} \lambda_j \|u_{G_j}\|_2\right]
  \left[\min_{j \in S_0} \|u_{G_j}\|_2/\lambda_j\right] \\
  \leq&
  \left[\sum_{j \notin S \cup S_0} \lambda_j \|u_{G_j}\|_2\right]
  \left[\sum_{j \in S_0} \lambda_j \|u_{G_j}\|_2/\sum_{j \in S_0} \lambda_j^2 \right] \\
  \leq&
  \frac{[\sum_{j \notin S} \lambda_j \|u_{G_j}\|_2]^2}{
    4 \lambda_-^2} .
  \end{align*}
This proves the first inequality of the lemma.
Similarly, we have
\begin{align*}
\sum_{k\geq 1} \|u_{G^k}\|_2 =&
\sum_{k\geq 1} \sqrt{\sum_{j \in S_k} \|u_{G_j}\|_2^2} \\
\leq&
\sum_{k\geq 1} \sqrt{\sum_{j \in S_k} \lambda_j \|u_{G_j}\|_2} \sqrt{\max_{j \in S_k} \|u_{G_j}\|_2/\lambda_j} \\
\leq&
\sum_{k\geq 1} \sqrt{\sum_{j \in S_k} \lambda_j \|u_{G_j}\|_2} \sqrt{\min_{j \in S_{k-1}} \|u_{G_j}\|_2/\lambda_j} \\
\leq&
\sum_{k\geq 1} \sqrt{\sum_{j \in S_k} \lambda_j \|u_{G_j}\|_2} \sqrt{\sum_{j \in S_{k-1}} \lambda_j |u_{G_j}\|_2/\sum_{j \in S_{k-1}}\lambda_j^2} \\
\leq& \lambda_-^{-1}
\sum_{k\geq 1} \sqrt{\sum_{j \in S_k} \lambda_j \|u_{G_j}\|_2} \sqrt{\sum_{j \in S_{k-1}} \lambda_j |u_{G_j}\|_2} \\
\leq& \lambda_-^{-1}
\sum_{k\geq 1} 
\frac{1}{2} \left[ \sum_{j \in S_k} \lambda_j \|u_{G_j}\|_2 +
\sum_{j \in S_{k-1}} \lambda_j |u_{G_j}\|_2 \right] \\
\leq& \lambda_-^{-1}
\sum_{k\geq 0} \sum_{j \in S_k} \lambda_j \|u_{G_j}\|_2
= \lambda_-^{-1} \sum_{j \notin S} \lambda_j \|u_{G_j}\|_2 .
\end{align*}
Therefore
  \begin{align*}
    n^{-1} \left|\sum_{j \notin S \cup S_0} u_G^\top X_G^\top X_{G_j} u_{G_j}\right|
  \leq& n^{-1} \sum_{k \geq 1} |u_G^\top X_G^\top X_{G^k} u_{G^k}|\\
  \leq& n^{-1} \sum_{k \geq 1}  \|X_G^\top X_{G^k}\|_2 \|u_{G^k}\|_2 \|u_G\|_2 \\
  \leq& \tilde{\rho}_+ \|u_G\|_2  \sum_{k\geq 1} \|u_{G^k}\|_2 \\
  \leq& \tilde{\rho}_+ \lambda_-^{-1}\|u_G\|_2
 \sum_{j \notin S} \lambda_j \|u_{G_j}\|_2 .
\end{align*}
Note that Lemma~\ref{lem:A2} is used to bound
$\|X_G^\top X_{G^k}\|_2$.
This proves the second inequality of the lemma.
\end{proof}

The following lemma shows that the 
group $L_1$-norm of the group Lasso estimator's  non-signal part 
is small (compared to the group $L_1$-norm of the parameter estimation
error in the signal part).
\begin{lemma}
Let $\Fr(\bb) \in G_S$ for some $S \subset \{1,\ldots,m\}$.
Assume that for all $j$:
\[
\lambda_j \geq 4 \rho_+(G_j)^{1/2} \|(X_{G_j}^\top X_{G_j})^{-1/2} X_{G_j}^\top \epsilon\|_2 /\sqrt{n}.
\]
Then the solution of (\ref{eq:group-lasso}) satisfies:
\[
\sum_{j \notin S} \lambda_j \left\|\hb_{G_j} \right\|_2
\leq 3 \sum_{j\in S} \lambda_j \|\bb_{G_j}-\hb_{G_j}\|_2 .
\]
\label{lem:L1}
\end{lemma}
\begin{proof}
The first order condition is:
\begin{equation}
2 X^\top X (\hb-\bb) - 2 X^\top \epsilon +
\sum_{j=1}^m \lambda_j n
\frac{\hb_{G_j}}{\left\|\hb_{G_j} \right\|_2}
= 0 .
\label{eq:first-order}
\end{equation}
By multiplying both sides by $(\hb-\bb)^\top$, we obtain
\[
0 \geq -2 (\hb-\bb)^\top X^\top X (\hb-\bb)
= -2 (\hb-\bb)^\top X^\top \epsilon +
\sum_{j=1}^m \lambda_j n
\frac{(\hb-\bb)_{G_j}^\top \hb_{G_j}}{\left\| \hb_{G_j} \right\|_2} .
\]
Therefore
\begin{align*}
& \sum_{j \notin S} \lambda_j
\left\|\hb_{G_j} \right\|_2 \\
\leq& \sum_{j\in S} \lambda_j \|\bb_{G_j}-\hb_{G_j}\|_2
+ 2 (\hb-\bb)^\top X^\top \epsilon/n \\
\leq& \sum_{j\in S} \lambda_j \|\bb_{G_j}-\hb_{G_j}\|_2
+ 2\sum_{j=1}^m \rho_+(G_j)^{1/2} \|(\hb-\bb)_{G_j}\|_2 \|(X_{G_j}^\top X_{G_j})^{-1/2} X_{G_j}^\top \epsilon\|_2 /\sqrt{n} \\
\leq& \sum_{j\in S} \lambda_j \|\bb_{G_j}-\hb_{G_j}\|_2
+ 0.5 \sum_{j=1}^m \lambda_j \|(\hb-\bb)_{G_j}\|_2 .
\end{align*}
Note that the last inequality follows from
the assumption of the lemma.
By simplifying the above inequality, we obtain the desired bound.
\end{proof}

The following lemma bounds parameter estimation error by combining
the previous two lemmas.
\begin{lemma} \label{lem:group}
Let $\Fr(\bb) \in G_S$ for some $S \subset \{1,\ldots,m\}$.
  Consider $\ell \geq 1$ and let $s=|G_S| + \ell + k_0 -1$.
  Define
  \begin{align*}
  \lambda_-^2 =& \min \left\{ \sum_{j \in S'} \lambda_j^2 : |G_{S'}| \geq \ell\right\} , \\
\tilde{\rho}_+ =& \sqrt{(\rho_+(s)-\rho_-(2s-|G_S|))
(\rho_+(s-|G_S|)- \rho_-(2s-|G_S|))} .
\end{align*}
If for all $j$:
\[
\lambda_j \geq 4 \rho_+(G_j)^{1/2} \|(X_{G_j}^\top X_{G_j})^{-1/2} X_{G_j}^\top \epsilon\|_2 /\sqrt{n} ,
\]
and
\[
6 \frac{\tilde{\rho}_+}{\rho_-(s)}
 \leq \frac{\lambda_-}{\sqrt{\sum_{j \in S} \lambda_j^2}} ,
\]
then the solution of (\ref{eq:group-lasso}) satisfies:
\[
\|(\hb-\bb)\|_2 \leq \frac{1.5}{\rho_-(s)}
\left(1+ 1.5 \lambda_-^{-1}\sqrt{\sum_{j \in S} \lambda_j^2}\right)
 \sqrt{\sum_{j \in S} \lambda_j^2} .
\]
\end{lemma}
\begin{proof}
  Define $S_0$ as in Lemma~\ref{lem:rip}.
  Let $G=\cup_{j \in S \cup S_0} G_j$.
  By multiplying both sides of (\ref{eq:first-order})
  by $(\hb-\bb)_G^\top$, we obtain
\[
2(\hb-\bb)_G^\top X_G^\top X (\hb-\bb) -
2(\hb-\bb)_G^\top X^\top \epsilon +
\sum_{j \in S \cup S_0} \lambda_j n
\frac{(\hb-\bb)_{G_j}^\top \hb_{G_j}}{\left\|\hb_{G_j} \right\|_2}
= 0 .
\]
Similar to the proof in Lemma~\ref{lem:L1}, we
use the assumptions on $\lambda_j$ to obtain:
\begin{equation}
4 n^{-1} (\hb-\bb)_G^\top X_G^\top X (\hb-\bb)
+ \sum_{j \in S_0} \lambda_j \left\|\hb_{G_j} \right\|_2
\leq 3 \sum_{j \in S} \lambda_j \|\hb_{G_j}-\bb_{G_j}\|_2 .
\label{eq:2norm-bound}
\end{equation}
Now, Lemma~\ref{lem:rip} implies that
\[
 (\hb-\bb)_G^\top X_G^\top X (\hb-\bb)
\geq (\hb-\bb)_G^\top X_G^\top X_G (\hb-\bb)_G
- \tilde{\rho}_+ \lambda_-^{-1} n \|(\hb-\bb)_G\|_2
  \sum_{j \notin S} \lambda_j \|(\hb-\bb)_{G_j}\|_2 .
\]
By applying Lemma~\ref{lem:L1}, we have
\begin{align*}
n^{-1} (\hb-\bb)_G^\top X_G^\top X (\hb-\bb)
\geq& \rho_-(G) \|(\hb-\bb)_G\|_2^2
- 3 \tilde{\rho}_+ \lambda_-^{-1}\|(\hb-\bb)_G\|_2
  \sum_{j \in S} \lambda_j \|(\hb-\bb)_{G_j}\|_2 \\
\geq& \rho_-(G) \|(\hb-\bb)_G\|_2^2
- 3 \tilde{\rho}_+ \lambda_-^{-1} \sqrt{\sum_{j \in S} \lambda_j^2}
\|(\hb-\bb)_G\|_2^2 \\
\geq& 0.5 \rho_-(G) \|(\hb-\bb)_G\|_2^2 .
\end{align*}
The assumption of the lemma is used to derive the last inequality.
Now plug this inequality into (\ref{eq:2norm-bound}), we have
\[
\|(\hb-\bb)_G\|_2^2 \leq
1.5 \rho_-(G)^{-1} \sum_{j \in S} \lambda_j \|\hb_{G_j}-\bb_{G_j}\|_2
\leq 1.5  \rho_-(G)^{-1} \sqrt{\sum_{j \in S} \lambda_j^2}
\|(\hb-\bb)_G\|_2 .
\]
This implies
\[
\|(\hb-\bb)_G\|_2^2 \leq
 2.25  \rho_-(G)^{-2} \sum_{j \in S} \lambda_j^2 .
\]
Now Lemma~\ref{lem:rip} and Lemma~\ref{lem:L1} imply that
\begin{align*}
\|(\hb-\bb)\|_2^2 -
\|(\hb-\bb)_G\|_2^2
\leq& 0.25 \lambda_-^{-2} \left[\sum_{j \notin S} \lambda_j \|(\hb-\bb)_{G_j}\|_2\right]^2\\
\leq& 2.25 \lambda_-^{-2} \left[\sum_{j \in S} \lambda_j \|(\hb-\bb)_{G_j}\|_2\right]^2\\
\leq& 2.25 \lambda_-^{-2} \sum_{j \in S} \lambda_j^2 \|(\hb-\bb)_{G}\|_2^2 .
\end{align*}
By combining the previous two displayed inequalities, we obtain
the lemma.
\end{proof}

\section{Proof of Theorem~\ref{thm:group}}

Assumption~\ref{assump:noise} implies that with probability larger than
$1-\eta$, uniformly for all groups $j$, we have
\[
\|(X_{G_j}^\top X_{G_j})^{-0.5} X_{G_j}^\top (\epsilon- \rE \epsilon) \|_2
 \leq a \sqrt{k_j} + b \sqrt{\ln (m/\eta)} .
\]
It follows that with the choice of $A$, $B$, and $\lambda_j$,
$\lambda_j \geq 4 \rho_+(G_j)^{1/2} \|(X_{G_j}^\top X_{G_j})^{-1/2} X_{G_j}^\top \epsilon\|_2 /\sqrt{n}$
for all $j$.
Moreover, assumptions of the theorem also imply that
$\tilde{\rho}_+ \leq \rho_+(s)-\rho_-(2s)$,
and
\[
\frac{\tilde{\rho}_+}{\rho_-(s)} \leq
\frac{\rho_+(s) - \rho_-(2s)}{\rho_-(s)}
\leq c
\leq  \frac{\sqrt{\ell A^2 + g_\ell B^2}}{6\sqrt{2(k A^2 +g B^2)}}
\leq \frac{\lambda_-}{6\sqrt{\sum_{j \in S} \lambda_j^2}} .
\]
Note that we have used
$\sum_{j \in S'} [A^2 k_j + B^2] \leq
n \sum_{j \in S'} \lambda_j^2 \leq 2 \sum_{j \in S'} [A^2 k_j + B^2]$.

Therefore the conditions of Lemma~\ref{lem:group} are satisfied.
Its conclusion implies that
\begin{align*}
\|(\hb-\bb)\|_2
\leq& \frac{1.5}{\rho_-(s)}
\left(1+ 1.5 \lambda_-^{-1}\sqrt{\sum_{j \in S} \lambda_j^2}\right)
 \sqrt{\sum_{j \in S} \lambda_j^2} \\
\leq& \frac{1.5}{\rho_-(s)}
\left(1+ \frac{1}{4c}\right)
 \sqrt{\sum_{j \in S} \lambda_j^2} \\
\leq& \frac{1.5}{\rho_-(s)}
\left(1+ \frac{1}{4c}\right)
 \sqrt{2 (A^2 k + B^2 g)/n} .
\end{align*}
This proves the theorem.

\end{document}